%% file: main.tex
\documentclass{article}

     \PassOptionsToPackage{numbers, compress}{natbib}

\usepackage{PRIMEarxiv}

\usepackage[utf8]{inputenc} 
\usepackage[T1]{fontenc}    
\usepackage{hyperref}       
\usepackage{url}            
\usepackage{booktabs}       
\usepackage{amsfonts}       
\usepackage{nicefrac}       
\usepackage{microtype}      
\usepackage{xcolor}         
\usepackage{wrapfig}
\usepackage[american]{babel}

\usepackage{scabby}

\usepackage{algorithm}
\usepackage[noend]{algcompatible}
\usepackage{subfigure}
\usepackage{float}
\usepackage{multirow}
\usepackage{graphicx}
\usepackage{subfigure}

\providecommand{\e}{\epsilon}

\newcommand{\ellavg}{{\bar \ell}}

\newcommand{\Ic}{{\mathcal I}}

\newcommand{\Fc}{{\mathcal S}}

\newcommand{\Rb}{{\mathbb R}}
\newcommand{\extR}{R^\mathrm{ext}}
\newcommand{\intR}{R^\mathrm{int}}
\newcommand{\Fcal}{S^\mathrm{cal}}
\newcommand{\Ind}{{\mathbb{I}}}
\newcommand{\Exp}{{\mathbb{E}}}

\newcommand{\palg}{F^H}
\newcommand{\wsupj}{\Ind^{(j)}}

\newcommand{\rjt}{{\rho^{(j)}_T}}

\newtheorem{defn}{Definition}

\newtheorem{fact}{Fact}

\usepackage{natbib} 
\usepackage{mathtools} 
\usepackage{booktabs} 
\usepackage{tikz} 

\title{Calibrated Regression Against An Adversary \\Without Regret}

\author{
  Shachi Deshpande \\
  Department of Computer Science \\
  Cornell University and Cornell Tech \\
  \texttt{ssd86@cornell.edu} \\
  \And
  Charles Marx \\
  Department of Computer Science \\
  Stanford University \\
  \texttt{ctmarx@stanford.edu } \\
  \And
  Volodymyr Kuleshov \\
  Department of Computer Science \\
  Cornell University and Cornell Tech \\
\texttt{kuleshov@cornell.edu} \\
}
  
\begin{document}
\maketitle




\begin{abstract}
We are interested in probabilistic prediction in online settings in which data does not follow a probability distribution. Our work seeks to achieve two goals: (1) producing valid probabilities that accurately reflect model confidence; (2) ensuring that traditional notions of performance (e.g., high accuracy) still hold. 
We introduce online algorithms guaranteed to achieve these goals on arbitrary streams of datapoints, including data chosen by an adversary.
Specifically, our algorithms produce forecasts that are (1) calibrated---i.e., an 80\% confidence interval contains the true outcome 80\% of the time---and (2) have low regret relative to a user-specified baseline model.
We implement a post-hoc recalibration strategy that provably achieves these goals in regression; previous algorithms applied to classification or achieved (1) but not (2).
In the context of Bayesian optimization, an online model-based decision-making task in which the data distribution shifts over time, our method yields accelerated convergence to improved optima. 
\end{abstract}

\section{Introduction}\label{sec:introduction}

In many applications of machine learning (ML), data can change over time. Online learning algorithms can guarantee good predictive accuracy (e.g., as measured by squared error) on arbitrary data streams, even ones chosen adversarially \citep{cesabianchi2006prediction,shalev2007phd}. 
However, we are often interested not only in minimizing predictive error, but also in outputting valid probabilities representative of future outcomes
\citep{vovk2005defensive, kuleshov2018accurate, angelopoulos2021gentle}.
For example a doctor might wish to estimate the probability of a patient being sick; similarly, a power grid operator might want to know the likelihood that demand for electricity will increase. 

In this paper, we are interested in probabilistic predictions in online settings where data does not follow a probability distribution \citep{shalev2007phd}. This setting is challenging because we need to achieve two goals on data that shifts over time: (1) producing valid probabilities that accurately reflect model confidence; (2) ensuring that traditional notions of performance (e.g., achieving a low squared error) still hold. Additionally, without a data distribution, these goals may not be straightforward to define.

Our approach towards the first goal uses calibration to define valid probabilistic forecasts \citep{foster98asymptoticcalibration,kuleshov2017estimating,gibbs2022conformal}. Intuitively, an algorithm outputs calibrated predictions if the predicted and the empirical probabilities of a predicted outcome match---i.e., an 80\% confidence interval contains the true outcome 80\% of the time. 
We formalize the second goal by requiring that calibrated predictions have low regret relative to a baseline uncalibrated forecaster, as measured by a proper score \citep{gneiting2007probabilistic}.
We focus on real-valued outcomes, and define online calibrated regression, a task that seeks to achieve the above two goals.

We propose algorithms for online calibrated regression that output accurate probabilistic predictions 
via the post-hoc recalibration of a black-box baseline model.
Unlike classical recalibration methods \citep{platt1999probabilistic,kuleshov2018accurate}, ours work on online non-IID data (even data chosen by an adversary). In contrast to classical online learning \citep{shalev2007phd}, we provide guarantees on not only regret, but also on the validity of probabilistic forecasts.
Crucially, unlike many online calibrated and conformal prediction algorithms for classification \citep{foster98asymptoticcalibration} or regression \citep{gibbs2022conformal}, we ensure low regret relative to a baseline forecaster. 


Accurate predictive uncertainties can be especially useful in decision-making settings, where an agent uses a model of future outcomes to estimate the results of its actions (e.g., the likelihood of treating a patient) \citep{Malik2019Calibrated}.
We complement our algorithms with formal guarantees on expected utility estimation in decision-making applications. 
We apply our algorithms to several regression tasks, as well in the context of Bayesian optimization, an online model-based decision-making task in which the data distribution shifts over time. We find that improved uncertainties in the Bayesian optimization model yield faster convergence to optimal solutions which are also often of higher quality. 

\paragraph{Contributions.} 
%
First, we formulate a new problem called online calibrated regression, which requires producing calibrated probabilities on potentially adversarial input while retaining the predictive power of a given baseline uncalibrated forecaster.
Second, we propose an algorithm for this task that generalizes recalibration in regression to non-IID data. 
Third, we show that the algorithm can improve the performance of Bayesian optimization, highlighting its potential to improve decision-making.

\section{Background}\label{sec:background}




We place our work in the framework of online learning 
\citep{shalev2007phd}.
At each time step $t = 1,2,...$, we are given features $x_t \in \mathcal{X} $. We use a forecaster $H : \mathcal{X} \to \mathcal{F}$ to produce a forecast $f_t = H(x_t)$, $f_t \in \mathcal{F}$ in a set of forecasts $\mathcal{F}$ over a target $y \in \mathcal{Y}$.
Nature then reveals the true target $y_t \in \mathcal{Y}$ and we incur a loss of $\ell(y_t, f_t)$, where $\ell : \mathcal{Y} \times \mathcal{F}  \to \Rb^+$ is a loss function.
%
Unlike in classical machine learning, we do not assume that the $x_t, y_t$ are i.i.d.: they can be random, deterministic, or even chosen by an adversary. In this regime, online learning algorithms admit strong performance guarantees measured in terms of regret $R_T(g)$ relative to a constant prediction $g$, $R_T(g) = \sum_{t=1}^T \ell(y_t, f_t)  - \ell(y_t, g). $
The worst-case regret at time $T$ equals $R_T = \max_{g \in \mathcal{F}} R_T(g)$.

\paragraph{Online forecasting}
\label{sec:forecasting}

Our work extends the online learning setting to probabilistic predictions.
 We focus on regression, where $y_t \in \mathbb{R}$ and the prediction $f_t$ can be represented by a cumulative distribution function (CDF), which we denote by $F_t : \mathbb{R} \to [0,1]$; $F_t(z)$ denotes the predicted probability that $y$ is less than $z$.
The quality of probabilistic forecasts is evaluated using {\em proper} losses $\ell$. Formally, 
a loss $\ell(y, f)$ is proper if
$f \in \arg\min_{g \in \mathcal{F}} \Exp_{y \sim (f)} \ell(y, g) \; \forall f \in \mathcal{F}.$; i.e., the true data probability minimizes the loss.
An important proper loss for CDF predictions is the continuous ranked probability score, defined as
$\ell_\text{CRPS}(y, F) = \int_{-\infty}^\infty (F(z) - \mathbb{I}_{y \leq z})^2 dz.$





\paragraph{Online calibration}

Proper losses decompose into a calibration and a sharpness component: these quantities precisely define an ideal forecast.
Intuitively, calibration means that a 60\% prediction should be valid 60\%  of the time; sharpness means that confidence intervals should be tight.

In the online setting, there exist algorithms guaranteed to produce calibrated forecasts of binary outcomes $y_t \in \{0,1\}$ even when the $y_t$ is adversarial \cite{foster98asymptoticcalibration,cesabianchi2006prediction,abernethy11blackwell}.
These algorithms
are oftentimes randomized; hence their guarantees hold almost surely (a.s.). Here, and in all other usages going forward, ``almost surely'' refers to the simulated randomness in the randomized algorithm, and not the data. 
However, most calibration methods do not account for covariates $x_t$ \cite{foster98asymptoticcalibration} or assume simple binary $y_t$ \cite{kuleshov2017estimating,foster2022calibeating}. We extend this work to regression and add guarantees on regret.

\section{Online calibrated regression}\label{sec:recalibration}

Next, we define a task in which our goal is to produce calibrated forecasts in a regression setting while maintaining the predictive accuracy of a baseline uncalibrated forecaster.



We start with a forecaster $H$ (e.g., an online learning algorithm) that outputs uncalibrated forecasts $F_t$ at each step; these forecasts are fed into a {\em recalibrator} such that the resulting forecasts $G_t$ are calibrated and have low regret relative to the baseline forecasts $F_t$. 
Formally, we introduce the setup of {\em online recalibration}, in which at every step $t=1,2,...$ we have:
  \begin{algorithmic}[1]
    \STATE Nature reveals features $x_t \in \mathbb R^d$. Forecaster $H$ predicts $F_t = H(x_t)$
    \STATE A recalibration algorithm produces a calibrated forecast $G_t$ based on $F_t$.
    \STATE Nature reveals continuous label $y_t \in \mathcal{Y} \subseteq \mathcal{R}$ bounded by $|y_t| < B/2$, where $B>0$.
    \STATE Based on $x_t, y_t$, we update the recalibration algorithm and optionally update $H$.
  \end{algorithmic}
  


Our task is to produce calibrated forecasts. Intuitively, we say that a forecast $F_t$ is calibrated if for every $y' \in \mathcal{Y}$, the probability $F_t(y')$ on average matches the frequency of the event $\{ y \leq y' \}$---in other words the $F_t$ behave like calibrated CDFs.
%
We formalize this intuition by introducing the ratio 
\begin{equation}
\label{eq:empirical_cal}
    \rho_T(y, p) = \dfrac{\sum_{t=1}^T \Ind_{y_t \leq y, F_t(y) = p}}{\sum_{t=1}^T \Ind_{F_t(y) = p}}.
\end{equation}
 Intuitively, we want $ \rho_T(y, p) \to p, $ as $T \to \infty$ for all $y$.
In other words, out of the times when the predicted probability $F_t(y')$ for $\{y_t \leq y'\}$ to be $p$, the event $\{y_t \leq y'\}$ holds a fraction $p$ of the time. We define $ \rho_T(y, p)$ to be zero when the denominator in Equation \eqref{eq:empirical_cal} is zero. Below, we enforce that $\rho_T(y, p) \to p$ for forecasts $p$ that are played infinitely often, in that $\sum_{t=1}^T \Ind_{F_t(y) = p} \to \infty$; if a forecast ceases to be played, there is no need (or opportunity) to improve calibration for that forecast.

We measure calibration using an extension of the aforementioned calibration error $C_T$. 
We define the calibration error of forecasts $\{F_t\}$ as 
\begin{equation}
\label{eq:miscal_error}
    C_T(y)  = \sum_{p \in P_T(y)} \left| \rho_T(y,p) - p \right| \left( \frac{1}{T} \sum_{t=1}^T \Ind_{\{F_t(y) = p\}} \right), 
\end{equation}
where $P_T(y) = \{F_1(y), F_2(y), ..., F_T(y)\}$ is the set of previous predictions for $\{y_t \leq y\}$.
To measure (mis)calibration for the recalibrated forecasts $G_t$, we replace $F_t$ with $G_t$ in Equation \eqref{eq:miscal_error}.
\begin{defn}
A sequence of forecasts $G_t$ is $\epsilon$-calibrated for $y \in \mathcal{Y}$ if $C_T(y) \leq R_T + \epsilon$ for $R_T = o(1)$, where $R_T$ represents the convergence rate.
\end{defn}

The interpretation of $\epsilon$-calibration is simple: for example, if $\epsilon = 0.01$, then of the times when we predict a 90\% chance of rain, the observed occurrence of rain will be between 89\% and 91\%. For most applications, an error tolerance of a few \% is acceptable. Note that the use of an error tolerance $\epsilon$ mirrors previous works \citep{foster98asymptoticcalibration,abernethy11blackwell,kuleshov2017estimating}.

The goal of recalibration is also to produce forecasts that have high predictive value \citep{gneiting2007probforecast}. 
We enforce this by requiring that the $G_t$ have low regret relative to the baseline $F_t$ in terms of the CRPS proper loss. 
Since the CRPS is a sum of calibration and sharpness terms, by maintaining a good CRPS while being calibrated, we effectively implement Gneitig's principle of maximizing sharpness subject to calibration \citep{gneiting2007probabilistic}.
Formally, this yields the following definition.
\begin{defn}
A sequence of forecasts $G_t$ is $\e$-recalibrated relative to forecasts $F_t$ if (a) the forecasts  $G_t$ are $\epsilon$-calibrated for all $y \in \mathcal{Y}$ and (b) the regret of $G_t$ with respect to $F_t$ is a.s.~small w.r.t.~$\ell_\textrm{CRPS}$:  
\begin{align*}
\lim\sup_{T \to \infty} \frac{1}{T} \sum_{t=1}^T \left( \ell_\textrm{CRPS}(y_t , G_t) - \ell_\textrm{CRPS}(y_t, F_t)\right) \leq \epsilon.
\end{align*}
\end{defn}

%
%

\section{Algorithms for online regression}\label{sec:framework}

\begin{figure}
\vspace{-3mm}
\begin{algorithm}[H]
  \caption{Online recalibration}
  \label{algo:recal}
  \begin{algorithmic}[1]
    \REQUIRE Online binary calibration subroutine $\Fcal$ with resolution $N$; number of intervals $M$
    \STATE Initialize $\Ic = \{[0,\frac{1}{M}), [\frac{1}{M}, \frac{2}{M}), ..., [\frac{M-1}{M},1]\}$, a set of intervals that partition $[0,1]$.
    \STATE Initialize $\Fc = \{ \Fcal_j \mid j = 0,...,M-1 \}$, a set of $M$ instances of $\Fcal$, one per $I_j \in \mathcal{I}$.
    \FOR {$t=1,2,...$:}
    \STATE Observe uncalibrated forecast $F_t$.
    \STATE Define $G_t(z)$ as the output of $\Fcal_{\lfloor F_t(z) \rfloor}$, where $\lfloor F_t(z) \rfloor$ is the index of the subroutine associated with the interval containing $F_t(z)$.
    \STATE Output $G_t$. Observe $y_t$ and update recalibrator:
        \FOR {$j=1,2,...,M$:}
        		\STATE $o_{tj} = 1 \text{ if } F(y_t) \leq \frac{j}{M} \text{ else } 0$. Pass $o_{tj}$ to $\Fcal_j$.
	\ENDFOR
    \ENDFOR
  \end{algorithmic}
\end{algorithm}
\vspace{-7mm}
\end{figure}

Next, we propose an algorithm for performing online recalibration (\algorithmref{recal}). This algorithm sequentially observes uncalibrated CDF forecasts $F_t$ and returns forecasts $G_t$ such that $G_t(z)$ is a calibrated estimate for the outcome $y_t \leq z$. This algorithm relies on a classical calibration subroutine (e.g., \citet{foster98asymptoticcalibration}), which it uses in a black-box manner to construct $G_t$.

\algorithmref{recal} can be seen as producing a $[0,1] \to [0,1]$ mapping that remaps the probability of each $z$ into its correct value.
More formally, \algorithmref{recal} 
partitions $[0,1]$ into $M$ intervals $\Ic = \{[0,\frac{1}{M}), [\frac{1}{M}, \frac{2}{M}), ..., [\frac{M-1}{M},1]\}$; each interval is associated with an instance $\Fcal$ of a binary calibration algorithm (e.g., \citet{foster98asymptoticcalibration}; see below). In order to compute $G_t(z)$, we compute $p_{tz} = F_t(z)$ and invoke the subroutine $\Fcal_j$ associated with interval $I_j$ containing $p_{tz}$.
After observing $y_t$, each $\Fcal_j$ observes the binary outcome $o_{tj} = \mathbb{I}_{F_t(y_t) \leq \frac{j}{M}}$ and updates itself.

\subsection{Online binary calibration subroutines}

A key component of \algorithmref{recal} is the binary calibration subroutine $\Fcal$.
This subroutine is treated as a black box, hence can implement a range of known algorithms including regret minimization \citep{foster98asymptoticcalibration, cesabianchi2006prediction}, Blackwell approchability \citep{abernethy11blackwell} or defensive forecasting \citep{vovk2005defensive}.
More formally, let $p_{tj}$ denote the output of the $j$-th calibration subroutine $\Fcal_j$ at time $t$. For any $p \in [0,1]$, we define $\rho_T^{(j)}(p) = {(\sum_{t=1}^T o_{tj} \Ind_{p_{tj} = p})}/{(\sum_{t=1}^T \Ind_{p_{tj} = p})}$ to be the empirical frequency of the event $\{ o_{tj} = 1\}$
Online calibration subroutines ensure that $\rho_T^{(j)}(p) \approx p$.

\paragraph{Assumptions.}

Specifically, 
a subroutine $\Fcal_j$ normally
outputs a set of discretized probabilities $i/N$ for $i \in \{0,1,...,N\}$. We refer to $N$ as their resolution.
We define the calibration error of $\Fcal_j$ at $i/N$ as 
$
C^{(j)}_{T,i}  = \left| \rjt(i/N) - \frac{i}{N} \right| \left( \frac{1}{T} \sum_{t=1}^T \wsupj_{t,i} \right)
$
where $\wsupj_{t,i} = \Ind\{p_{tj} = i/N\}$. 
We may write the calibration loss of $\Fcal_j$ 
as $C^{(j)}_{T} = \sum_{i=0}^N C^{(j)}_{T,i}$.

We will assume that the subroutine $\Fcal$ used in \algorithmref{recal} is $\e$-calibrated in that $C^{(j)}_{T} \leq R_{T} + \e$ uniformly ($R_{T} = o(1)$ as $T \to \infty$). 
%
Recall also that the target $y_t$ is bounded as $|y_t| < B/2$.

\subsection{Online recalibration produces calibrated forecasts}

Intuitively, Algorithm \ref{algo:recal} produces valid calibrated estimates $G_t(z)$ for each $z$ because each $\Fcal_j$ is a calibrated subroutine. More formally, we seek to quantify the calibration of Algorithm \ref{algo:recal}. Since the $\Fcal$ output discretized probabilities, we may define the calibration loss of Algorithm \ref{algo:recal} at $y$ as
$$
C_{T}(y) = \sum_{i=0}^N \left|  \rho_T(y, i/N) - \frac{i}{N} \right| \left( \frac{1}{T} \sum_{t=1}^T \Ind_{t,i} \right),
$$
where $\Ind_{t,i} = \Ind\{F(y_t) = i/N\}$. 
The following lemma establishes
that combining the predictions of each $\Fcal_j$ preserves their calibration. 
Specifically, the calibration error of Algorithm \ref{algo:recal} is bounded by a weighted average of $R_{T_j}$ terms, each is $o(1)$, hence the bound is also $o(1)$ (see next section).

\begin{lemma}[Preserving calibration]\label{lem:calibration}
Given $y \in \mathcal{Y}$, let $T_j = |\{ 1 \leq t \leq T : \lfloor F_t(y) \rfloor = j/M \}|$ denote the number of calls to $\Fcal_j$ by \algorithmref{recal}.
If each $\Fcal_j$ is $\e$-calibrated,
then \algorithmref{recal} is also $\e$-calibrated and the following bound holds uniformly a.s. over $T$:
$$C_T(y) \leq \sum_{j=1}^M \frac{T_j}{T} R_{T_j} + \e$$ 
\end{lemma}

\subsection{Online recalibration produces forecasts with vanishing regret}


%

Next, we want to show that the $G_t$ do not decrease the predictive performance of the $F_t$, as measured by $\ell_\text{CRPS}$. Intuitively, this is true because the $\ell_\text{CRPS}$ is a proper loss that is the sum of calibration and sharpness, the former of which improves in $G_t$. 

Establishing this result will rely on the following key technical lemma~\citep{kuleshov2017estimating} (see Appendix).

\begin{lemma}\label{lem:noregret}
Each $\e$-calibrated $\Fcal_j$
a.s.~has a small regret w.r.t.~the $\ell_2$ norm and satisfies uniformly over time $T_j$ the bound
$
\max_{i,k} \sum_{t=1}^{T_j} \Ind_{p_{tj} = i/N} \left( \ell_2(o_{tj}, i/N)  - \ell_2(o_{tj}, k/N) \right) \leq 2 (R_{T_j} + \e).
$
\end{lemma}



An important consequence of \lemmaref{noregret} is that a calibrated algorithm has vanishing regret relative to any fixed prediction (since minimizing internal regret also minimizes external regret). Using this fact, it becomes possible to establish that \algorithmref{recal} is at least as accurate as the baseline forecaster. 

\begin{lemma}[Recalibration with low regret accuracy]\label{lem:accuracy}
Consider \algorithmref{recal} with
parameters $M \geq N > 1/\e$ and let $\ell$ be the CRPS proper loss. 
Then the recalibrated $G_t$ a.s.~have vanishing $\ell$-loss regret relative to $F_t$ and we have a.s.:
\begin{equation*}
\frac{1}{T} \sum_{t=1}^T \ell (y_t , G_t) - \frac{1}{T} \sum_{t=1}^T \ell(y_t , F_t) <NB R_T + \frac{2B}{N} 
\end{equation*}
\end{lemma}

\begin{proof}[Proof (sketch)]
When $p_{tj} = G_t(y)$ is the output of a given binary calibration subroutine $\Fcal_j$ at some $y$, we know what $\lfloor F(y) \rfloor = j/M$ (by construction). Additionally, we know from \lemmaref{noregret} that $\Fcal_j$ minimizes regret. Thus, it has vanishing regret in terms of $\ell_2$ loss relative to the fixed prediction $j/M$: $\sum_{t=1}^{T_j} (o_{tj}-p_{tj})^2 \leq \sum_{t=1}^{T_j} (o_{tj}-j/M)^2 + o(T_j)$. But $o_{tj} = \Ind_{F(y_t) \leq j/m}$, and during the times $t$ when $\Fcal_j$ was invoked, during the times $t$ when $\Fcal_j$ was invoked $p_{tj} = G_t(y)$ and $j/M = F_t(y)$. Aggregating over $j$ and integrating over $y$ yields our result.
\end{proof}

These two lemmas lead to our main claim: that \algorithmref{recal} solves the online recalibration problem.

\begin{theorem}\label{thm:main}
Let $\Fcal$ be an $(\epsilon/2B)$-calibrated online subroutine with resolution $N \geq 2B/\epsilon$. 
Then \algorithmref{recal} with parameters $\Fcal$ and $M=N$ outputs $\epsilon$-recalibrated forecasts.
\end{theorem}

\begin{proof}
By \lemmaref{calibration}, \algorithmref{recal} is $(\e/2B)$-calibrated and by \lemmaref{accuracy}, its regret w.r.t. the $F_t$ tends to $< 2B/N < \e$. Hence, \theoremref{main} follows.
\end{proof}


\paragraph{General proper losses}

Throughout our analysis, we have used the CRPS loss to measure the regret of our algorithm. This raises the question: is the CRPS loss necessary? One answer to this question is that if the loss $\ell$ used to measure regret is not a proper loss, then recalibration is not possible. 

\begin{theorem}
If $\ell$ is not proper, then no algorithm achieves recalibration w.r.t.~$\ell$ for all $\e > 0$.
\end{theorem}




On the other hand, in Appendix \ref{app:regret}, we provide a more general analysis that shows that: (1) a calibrated $\Fcal$ must have vanishing regret relative to a fixed prediction as measured using any proper score; (2) \algorithmref{recal} achieves vanishing regret relative to any proper score.
See Appendix \ref{app:regret} for a formal statement and proof.

\section{Applications}

\subsection{Choice of recalibration subroutine}


\algorithmref{recal} is compabible with any binary recalibration subroutine $\Fcal$. Two choices of $\Fcal$ include methods based on {\bf internal regret minimization} \citep{mannor2010calibration} and ones based on {\bf Blackwell approachability} \citep{abernethy11blackwell}. These yield different computational costs and convergence rates for \algorithmref{recal}.

Specifically, recall that $R_T$ denotes the rate of convergence of the calibration error $C_T$ of \algorithmref{recal}.
For most online calibration subroutines $\Fcal$,
$R_{T} \leq f(\e)/\sqrt{T}$ for some $f(\e)$.
In such cases, we can further bound the calibration error in \lemmaref{calibration} as
$$
\sum_{j=1}^M \frac{T_j}{T} R_{T_j} \leq \sum_{j=1}^M \frac{\sqrt{T_j}f(\e)}{T} \leq \frac{f(\e)}{\sqrt{ \e T}}. 
$$
In the second inequality, we set the $T_j$ to be equal. 
Thus, our recalibration procedure introduces an overhead of
$ \frac{1}{\sqrt{\e}} $
in the convergence rate of the calibration error $C_T$ and of the regret in \lemmaref{accuracy}.
In addition, we require $ \frac{1}{{\e}} $ times more memory and computation time (we run $1/\e$ instances of $\Fcal_j$). 

When using an internal regret minimization subroutine, 
the overall calibration error of \algorithmref{recal} is bounded as $O({1}/{\e \sqrt{\e T}})$ with $O(1/\e)$ time and $O(1/\e^2)$ space complexity. These numbers improve to $O(\log(1/\e))$ time complexity for a $O({1}/{\e \sqrt{T}})$ calibration bound when using the method of \citet{abernethy11blackwell} based on Blackwell approachability. The latter choice is what we recommend.

\subsection{Uncertainty estimation}

We complement our results with ways in which \algorithmref{recal} can yield predictions for various confidence intervals. 
\begin{theorem}
    Let $G_t$ for $t=1,2,...,T$ denote a sequence of $(\epsilon/2)$-calibrated forecasts. For any interval $[y_1, y_2]$, we have $\frac{1}{T} \sum_{t=1}^T ( G_t(y_2) - G_t(y_1) ) \to \frac{1}{T} \sum_{t=1}^T \mathbb{I}\{y_t \in [y_1, y_2]\}$ as $T \to \infty$ a.s.
\end{theorem}

This theorem justifies the use of $F_t(y_2) - F_t(y_1)$ to estimate the probability of the event that $y_t$ falls in the interval $[y_1, y_2]$: on average, predicted probabilities will match true outcomes. The proof follows directly from the definition of 
$\epsilon$-calibration. This result directly mirrors the construction for calibrated confidence intervals in \citet{kuleshov2018accurate}.

\subsection{Online decision-making}

Consider a doctor seeing a stream of patients. For each patient $x_t$, they use a model $M$ of an outcome $y_t$ to estimate a loss $\ell(x_t) = \mathbb{E}_{y \sim M(x_t)} \ell(x_t, y, a(x_t))$ for a decision $a(x_t)$ (which could be $a(x_t) = \arg \min_a \mathbb{E}_{y \sim M(x_t)} [ \ell(x_t,y, a) ]$, e.g., a treatment that optimizes an expected outcome).
We want to guarantee that the doctor's predictions will be correct: over time, the estimated expected value will not exceed from the realized loss. Crucially, we want this to hold in non-IID settings.

Our framework enables us to achieve this result with only a weak condition---calibration.
The following concentration inequality shows that estimates of $v$ are unlikely to exceed the true $v$ on average (proof in Appendix \ref{app:applications}). If data was IID, this would be Markov's inequality: surprisingly, a similar  statement holds in non-IID settings.

\begin{theorem}
\label{thm:dist_calib_bound_app}
Let $M$ be a calibrated model and
let $\ell(y, a, x)$ be a monotonically non-increasing or non-decreasing loss in $y$.
Then for any sequence $(x_t, y_t)_{t=1}^T$ and $r > 1$, we have:
\begin{equation}
    \label{eqn:dist_calib_bound1}
    \lim_{T \to \infty} \frac{1}{T} \sum_{t=1}^T \mathbb{I} \left[ \ell(y_t, a(x_t), x_t) \geq r \ell(x_t)) \right] \leq 1 / r
\end{equation}
\end{theorem}

\section{Experiments}\label{sec:experiments}

Next, we evaluate \algorithmref{recal} on regression tasks as well as on Bayesian optimization, a sequential decision-making process that induces a non-i.i.d.~data distribution. We performed all  experiments on a laptop, indicating the low overhead of our method.

\paragraph{Baselines}

We compare our randomized online calibration with two baselines. Calibrated regression is a popular algorithm for the IID setting~\citep{kuleshov2018accurate} and can be seen as estimating the same mapping as Algorithm \ref{algo:recal} using kernel density estimation with a tophat kernel.  
Non-randomized online calibration uses the same subroutine as \algorithmref{recal}, but outputs the expected probability as opposed to a random sample; we found this to be a strong baseline that outperforms simple density estimation and reveals the value of randomization.

\textbf{Analysis of calibration.} We assess the calibration of the base model and the recalibrated model with calibration scores defined using the probability integral transform~\citep{gneiting2007probforecast}. 
We define the calibration score as 
$\text{cal}(p_1, y_1,..,p_n, y_n) = \sum_{j=1}^{m} ((q_j-q_{j-1}) - \hat{p_j})^2,$
where $q_0=0 < q_1 < q_2 <..<q_m = 1$ are $m$ confidence levels. The $\hat{p_j}$ is estimated as $\hat{p_j} = |\{ y_t |  q_{j-1} \leq p_t \leq q_j,  t=1,..,N\}|/N.$ 

\subsection{UCI Datasets}

We experiment with four multivariate UCI datasets~\citep{Dua2019UCI} to evaluate our online calibration algorithm.  

\textbf{Setup.}
Our dataset consists of input and output pairs $\{x_t, y_t\}_{t=1}^{T}$ where $T$ is the size of the dataset. 
We simulate a stream of data by sending batches of data-points $\{x_t, y_t\}_{t=nt'+1}^{n(t'+1)}$ to our model, where $t'$ is the time-step and $n$ is the batch-size. This simulation is run for $\left \lceil{T/n}\right \rceil $ time-steps. For each batch, Bayesian ridge regression is fit to the data and the recalibrator is trained. 
We set $N=20$ in the recalibrator and use a batch size of $n=10$ unless stated otherwise.

\begin{table*}[t]
  \caption{Evaluation of online calibration on UCI datasets. We compare the performance of online calibration against non-randomized online calibration, kernel density estimation, and uncalibrated (i.e., raw) baselines. Our method produces the lowest calibration errors in the last time step. Results hold with std error quoted in braces (10 experimental runs, fixed dataset).}
  \label{table:uci}
  \centering
  \begin{small}
    
     
    
  {
  \begin{tabular}{lcccc}
    \toprule
    Dataset & Uncalibrated  & Kernel Density & Online Calibration & Online Calibration \\
    
    & (Raw) & Estimation	& (Non-randomized) & \\
    \midrule
    Aq. Toxicity (Daphnia Magna) & {0.0081 (0.0001)} & {0.0055 (0.0002)} & {0.0058 (0.0003)} & {\textbf{0.0027 (0.0001)}} \\
    Aq. Toxicity (Fathead Minnow) & {0.0111 (0.0000)} & {0.0097 (0.0005)} & {0.0084 (0.0005)} &  {\textbf{0.0031 (0.0003)}} \\
     
      Energy Efficiency & 0.3322 (0.0001) & 0.2857 (0.0356) & 0.1702 (0.0094) & \textbf{0.1156 (0.0061)} \\
      Facebook Comment Volume & 0.2510 (0.0000) & 0.0589 (0.0050) & 0.0623 (0.0000) & \textbf{0.0518 (0.0002)} \\
    
    \bottomrule
  \end{tabular}
  }
   \end{small}
\end{table*}

\begin{figure*}[tb]
\centering     
\vspace{-3mm}
\subfigure[Aquatic Toxicity (Daphnia Magna)]{\label{fig:daphnia-aquatic-toxicity}\includegraphics[width=0.49\linewidth]{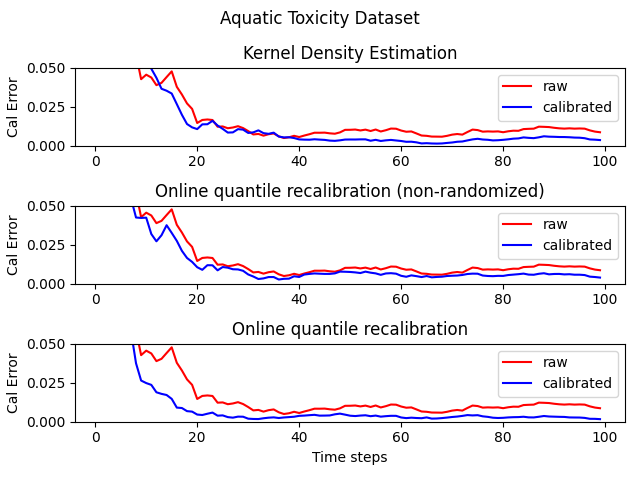}}
\subfigure[Aquatic Toxicity (Fathead Minnow)]{\label{fig:fathead-aquatic-toxicity}\includegraphics[width=0.49\linewidth]{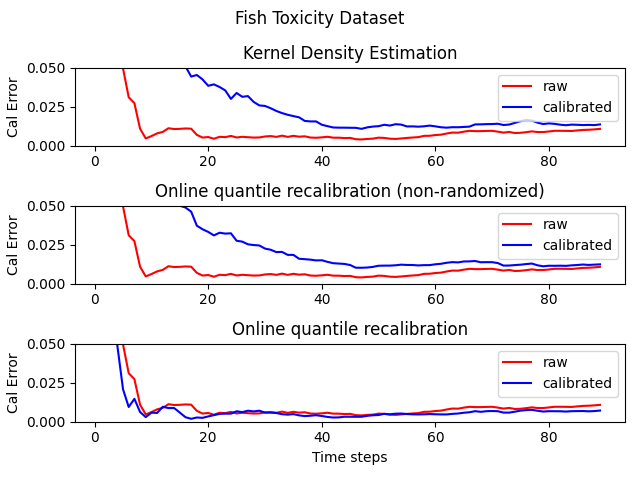}}

\caption{Performance of online calibration on the Aquatic Toxicity datasets. Aquatic toxicity towards two different types of fish (Daphnia Magna~\ref{fig:daphnia-aquatic-toxicity} and Fathead Minnow~\ref{fig:fathead-aquatic-toxicity}) is predicted by the base model. In both datasets, online calibration outperforms the baseline methods.}
\label{fig:aquatic-toxicity}
\end{figure*}

\paragraph{Aquatic toxicity datasets}
We evaluate our algorithm on the QSAR (Quantitative Structure-Activity Relationship) Aquatic Toxicity Dataset~\ref{fig:daphnia-aquatic-toxicity} (batch size n=5)
and Fish Toxicity Dataset~\ref{fig:fathead-aquatic-toxicity} (batch size n=10), where aquatic toxicity towards two different types of fish is predicted using 8 and 6 molecular descriptors as features respectively. 
In Figure~\ref{fig:aquatic-toxicity}, we can see that the randomized online calibration algorithm produces a lower calibration error than the non-randomized baseline. We also compare the performance of our algorithm against uniform kernel density estimation by maintaining a running average of probabilities in each incoming batch of data-points. For the Fish Toxicity Dataset, we can see that only online calibration improves calibration errors relative to the baseline model. We report all final calibration errors in Table~\ref{table:uci}. 

\begin{figure*}[!htb]
\centering     
\subfigure[Energy Efficiency]{\label{ewa-recalibrator-energy-efficiency}\includegraphics[width=0.49\linewidth]{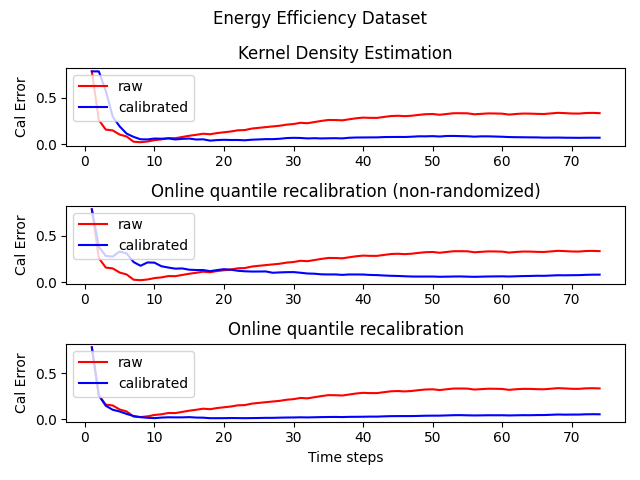}}
\subfigure[Facebook Comment Volume]{\label{ewa-recalibrator-facebook-comments}\includegraphics[width=0.49\linewidth]{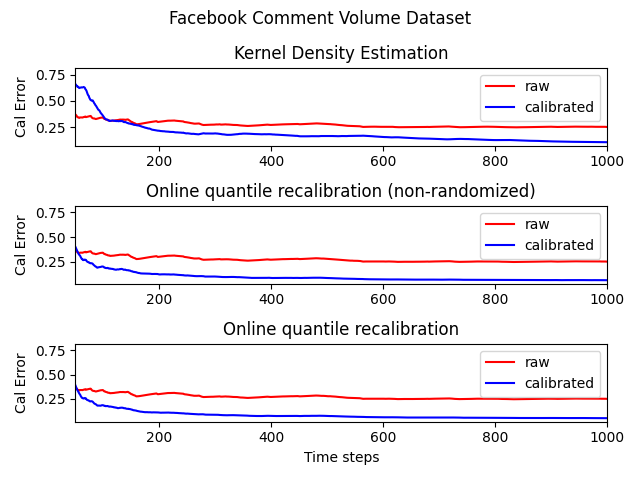}}

\caption{Performance of online calibration on the Energy Efficiency and Facebook Comment Volume datasets. In both datasets, online recalibration (blue, bottom) attains a lower calibration error at a faster rate than baselines (red and top, middle).}
\label{fig:other-datasets}
\vspace{-2mm}
\end{figure*}

\paragraph{Energy efficiency dataset}
The heating load and cooling load of a building is predicted using 8 building parameters as features. 
In Figure~\ref{ewa-recalibrator-energy-efficiency}, we see that the calibration errors produced by the online calibration algorithm drop sharply within the initial 10 time-steps. The baselines also produce a drop in calibration scores, but it happens more gradually.

\paragraph{Facebook comment volume dataset}
In Figure~\ref{ewa-recalibrator-facebook-comments}, the Facebook Comment Volume Dataset is used where the number of comments is to be predicted using 53 attributes associated with a post. We use the initial 10000 data-points from the dataset for this experiment. Here, the non-randomized and randomized online calibration algorithms produce a similar drop in calibration errors, but the  randomized online calibration algorithm still dominates both baselines (Table~\ref{table:uci}).

\subsection{Bayesian optimization}

We also apply online recalibration in the context of Bayesian optimization, an online model-based decision-making task in which {\bf the data distribution shifts over time} (it is the result of our actions). We find that improved uncertainties yield faster convergence to higher quality optima.

 \begin{wraptable}{l}{7.5cm}
  \caption{Recalibrated Bayesian optimization}
  \label{table:bayes-opt-uai}
  \centering
{
  \begin{tabular}{lcc}
    \toprule
    Benchmark & Uncalibrated & Recalibrated  \\
    
    \midrule
    Ackley (2D) & 9.925 (3.502)  & \textbf{8.313 (3.403)}   \\    
    SixHump (2D) & -0.378 (0.146) & \textbf{-1.029 (0.002)}  \\ 
    Ackley (10D) & 14.638 (0.591) &  \textbf{10.867 (2.343)}   \\    
      Alpine (10D) & 13.911 (1.846) &  \textbf{12.163 (1.555)}   \\      
    \bottomrule
  \end{tabular}
  }
\end{wraptable}

\begin{figure*}[h]
\centering     
\subfigure[SixHumpCamel]{\label{ewa-recalibrator-sixhumpcamel}\includegraphics[width=0.32\linewidth]{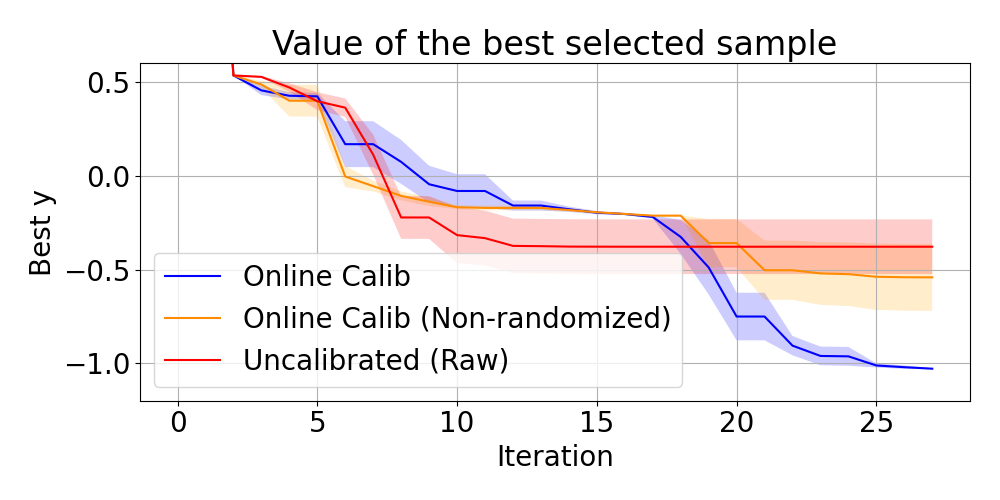}}
\subfigure[Beale]{\label{beale}\includegraphics[width=0.32\linewidth]{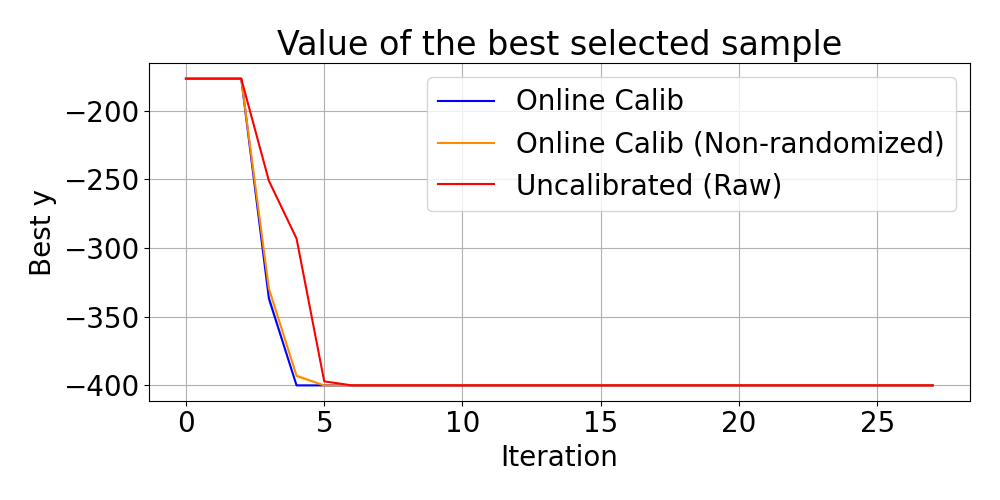}}
\subfigure[Mccormick]{\label{Mccormick}\includegraphics[width=0.32\linewidth]{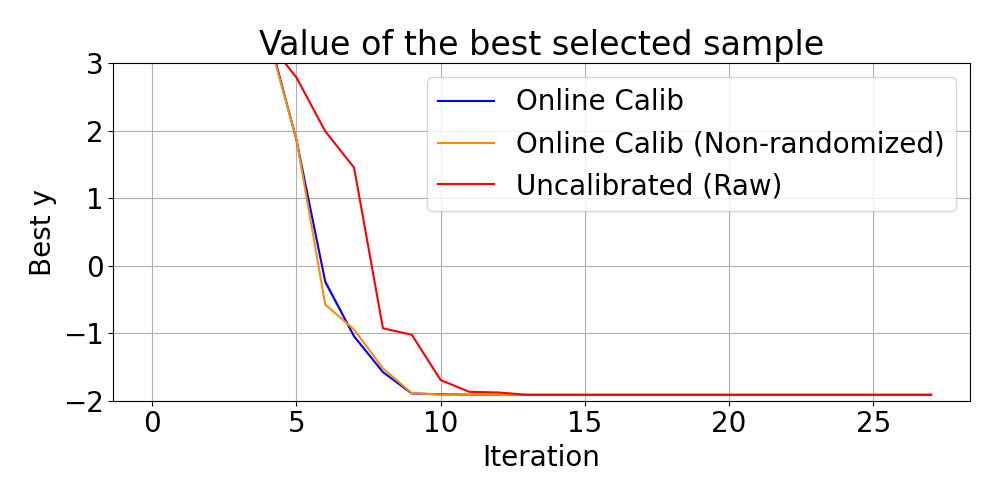}}
\caption{Performance of recalibration methods on Bayesian optimization benchmarks}
\vspace{-2mm}
\label{fig:bayes-opt-main}
\end{figure*}


\paragraph{Setup}
Bayesian optimization attempts to find the global minimum $x^\star = \arg \min_{x \in \mathcal{X}} f(x)$ of an unknown function $f:\mathcal{X} \to \mathbb{R} $ over an input space $\mathcal{X} \subseteq \mathbb{R}^D$. 
We are given an initial labeled dataset $x_t, y_t \in \mathcal{X} \times \mathbb{R}$ for $n=3$. At every time-step $t$, we use normal and recalibrated uncertainties from the probabilistic model $\mathcal{M}:\mathcal{X} \to (\mathbb{R} \to [0, 1])$ of $f$ (here, a Gaussian Process) to select the next data-point $x_{next}$ and iteratively update the model $\mathcal{M}$. 
We use popular benchmark functions to evaluate the performance of Bayesian optimization. 
We use the Lower Confidence Bound (LCB) acquisition function to select the data-point $x_t$.
See Appendix \ref{apdx:bayes_opt} for details.

Table \ref{table:bayes-opt-uai} shows that 
the online recalibration of uncertainties in a Bayesian optimization (BO) model
achieves lower minima than an uncalibrated model (results averaged over 5 overall BO runs with fixed initialization).
Figure~\ref{fig:bayes-opt-main} shows that online recalibrated Bayesian optimization can also reach optima in fewer steps. The error bars for the Beale and Mccormick functions are too small to be visible in the plots. All error bars denote standard errors. 

\begin{table*}[!h]
  \caption{Comparison to existing methods in the literature}
  \label{table:comparison}
  \centering
  \begin{small}
  {
  \begin{tabular}{llllllll}
    \toprule
    Method & Setting  & Output & Calibration & Recalibrator & Regret & Proof Technique \\
    \midrule
    Foster \& Vohra \cite{foster98asymptoticcalibration} &	Class.	& $p_t \in [0,1]$ & Conditional & n/a &	n/a & Int.~regret min. \\
    Kuleshov \& Ermon \cite{kuleshov2017estimating} &	Class.	& $p_t \in [0,1]$ & Conditional & $p$-to-$p$ & L2 loss & Int.~regret min. \\
    Gibbs \& Candes \cite{gibbs2022conformal} &	Regr.	& $q_t \in [0,1]$ & One quantile & $q$-to-$q$ & n/a & Quantile regr. \\
    Ours &	Regr.	& CDF $F_t$ & CDF $\forall y$ & $F(y)$-to-$F(y)$ & CRPS & CDF regr. \\
    \bottomrule
  \end{tabular}
  }
  \end{small}
\end{table*}

\section{Discussion}\label{sec:discussion}

\paragraph{Adversarial calibration methods}

Table \ref{table:comparison} compares our method against its closest alternatives. Unlike previous algorithms aimed at classification that output a binary forecast $p_t \in [0,1]$ \citep{foster98asymptoticcalibration,kuleshov2017estimating}, we study marginal quantile calibration in regression.
Our work resembles adaptive conformal inference \citep{gibbs2022conformal}, but provides a CDF-like object $F_t$ instead of one confidence interval $q_t \in [0,1]$ and yields a different notion of calibration. Crucially, we provide regret guarantees relative to a baseline model. 

Specifically, our technical goal is marginal CDF calibration: estimating the probability of the event 
$y_t \leq y$ for all $y$. Note that these probabilities are marginal over the $y_t$; this is in contrast to conditional calibration for $y_t = 1| p_t = p$ as in \citet{kuleshov2017estimating}.
We call our technical strategy online CDF regression (by analogy to quantile regression): we remap the predicted probabilities $F_t(y)$ (for any $y$) to a calibrated probability $R(F_t(y))$.
Our proof technique establishes calibration by relating final calibration to the calibration of each subroutine using Jensen’s inequality.
We establish low regret by aggregating the regret of all the subroutines within one CRPS loss.


Most existing methods in online calibrated classification \cite{foster98asymptoticcalibration,vovk2005defensive,abernethy11blackwell,okoroafor2024faster} or regression \cite{gibbs2022conformal} do not provide guarantees for regret, except online recalibrated classification \cite{kuleshov2017estimating} and calibeating \cite{foster2022calibeating,lee2022online}. However, these methods are only for binary classification, whereas ours are for regression.

\paragraph{Marginal calibration}
Our definition of calibration in regression is  marginal across all $x_t, y_t$; this is in contrast to classification \citep{foster98asymptoticcalibration}, where calibration is conditional (also known as distributional) on each $p$. 
Marginal calibration implies that the true outcome falls below the 90\% quantile 90\% of times (averaged over all $t$). 
%
Distribution calibration in regression \cite{Kuleshov2022Calibrated} would be PPAD-hard by reduction from multi-class \citep{hazan2012calibration}. Marginal calibration is also currently a common definition of calibration for regression. For example, \citet{kuleshov2018accurate} in the IID setting or \citet{gibbs2022conformal} in the online setting adopt this definition.

\paragraph{Batch vs online calibration}


\algorithmref{recal} can be seen as a direct counterpart to the histogram technique, a simple method for density estimation. With the histogram approach, the $F_t$ is split into N bins, and the average $y$ value is estimated for each bin. Because of the i.i.d. assumption, the output probabilities are calibrated, and the bin width determines the sharpness. Note that by Hoeffding's inequality, the average for a specific bin converges at a faster rate of $O({1}/{\sqrt{T_j}})$\citep{devroye1996probabilistic}, as opposed to the $O({1}/{\sqrt{\e T_j}})$ rate given by \citet{abernethy11blackwell}; hence online calibration is harder than batch. 




\section{Previous work \& conclusion}

Calibrated probabilities are widely used as confidence measures in the context of binary classification.
Such probabilities are obtained via recalibration methods, of which Platt scaling \cite{platt1999probabilistic} and isotonic regression \cite{niculescu2005predicting} are by far the most popular. Recalibration methods also possess multiclass extensions, which typically involve training multiple one-vs-all predictors \cite{zadrozny2002transforming}, as well as extensions to ranking losses \cite{menon2012ranking}, combinations of estimators \cite{zhong2013accurate}, and structured prediction \cite{kuleshov2015calibrated}. Recalibration algorithms have applied to improve reinforcement learning~\citep{Malik2019Calibrated}, Bayesian optimization~\citep{Deshpande2021Calibrated, stanton2023bayesian} and deep learning~\citep{Kuleshov2022Calibrated}. Crucially, all of these methods implicitly rely on the assumption that data is sampled i.i.d.~form an underlying distribution; they can be interpreted as density estimation techniques.

Online calibration was first proposed by \citep{foster98asymptoticcalibration}. Existing algorithms are based on internal regret minimization \cite{cesabianchi2006prediction} or on Blackwell approachability \cite{foster1997proof}; recently, these approaches were shown to be closely related \cite{abernethy11blackwell,mannor2010calibration}. 
%
Conformal prediction \citep{vovk2005defensive} is a technique for constructing calibrated predictive sets; it has been extended to handle distribution shifts \citep{hendrycks2018using, tibshirani2019conformal, Barber2022Conformal}, and online adversarial data \citep{gibbs2022conformal}. 


\paragraph{Conclusion}




We presented a novel approach to uncertainty estimation that leverages online learning. Our approach extends existing online learning methods to handle predictive uncertainty while ensuring high accuracy, providing formal guarantees on calibration and regret on adversarial input. 

We introduced a new problem called online calibrated forecasting, and proposed algorithms that generalize calibrated regression to non-IID settings. Our methods are effective on several predictive tasks and hold potential to improve performance in sequential model-based decision-making settings where we are likely to observe non-stationary data. 


\newpage

\bibliography{all}

\input{appendix}

\end{document}

%% file: appendix.tex
\newpage

\title{Adversarial Calibrated Regression
for Online Decision Making (Supplementary Material and Code)}

\appendix
\onecolumn

\section{Correctness of the recalibration procedure}

\label{app:proofs}

In the appendix, we provide the proofs of the theorems from the main part of the paper.

\paragraph{Notation}
We use $\Ind_E$ denote the indicator function of $E$, $[N]$ and $[N]_0$ to (respectively) denote the sets $\{1,2,...,N\}$ and $\{0,1,2,...,N\}$, and $\Delta_d$ to denote a $d$-dimensional simplex.

\paragraph{Setup}

We place our work in the framework of online learning 
\citep{shalev2007phd}.
At each time step $t = 1,2,...$, we are given features $x_t \in \mathcal{X} $. We use a forecaster $H : \mathcal{X} \to \mathcal{P}$ to produce a prediction $p_t = H(x_t)$, $p_t \in \mathcal{P}$ in the set of distributions $\mathcal{P}$ over a target $y \in \mathcal{Y}$.
Nature then reveals the true target $y_t \in \mathcal{Y}$ and we incur a loss of $\ell(y_t, p_t)$, where $\ell : \mathcal{Y} \times \mathcal{P}  \to \Rb^+$ is a loss function.
The forecaster $H$ updates itself based on $x_t, y_t$, and we proceed to time $t+1$.

Unlike in classical machine learning, we do not assume that the $x_t, y_t$ are i.i.d.: they can be random, deterministic or even chosen by an adversary. Online learning algorithms feature strong performance guarantees in this regime, where performance is usually measured in terms of regret $R_T(q)$ relative to a constant prediction $q$, $R_T(q) = \sum_{t=1}^T \ell(y_t, p_t)  - \ell(y_t, q). $
The worst-case regret at time $T$ equals $R_T = \max_{q \in \mathcal{P}} R_T(q)$.

In this paper, the predictions $p_t$ are probability distributions over the outcome $y_t$. We focus on regression, where $y_t \in \mathbb{R}$ and the prediction $p_t$ can be represented by a cumulative distribution function (CDF), denoted $F_t : \mathbb{R} \to [0,1]$ and defined as $F_t(z) = p_t(y \leq z)$.

\paragraph{Learning with expert advice}

A special case of this framework arises when each $x_t$ represents advice from $N$ {\em experts}, and $H$ outputs $p_t \in \Delta_{N-1}$, a distribution over experts. Nature reveals an outcome $y_t$, resulting in an expected loss of $\sum_{i=1}^N p_{ti} \ell(y_t, a_{ti})$, where $\ell(y_t, a_{ti})$ is the loss under expert $i$'s advice $a_{ti}$. Performance in this setting is measured using two notions of regret.
%
%
\begin{defn}
The external regret $\extR_T$ and the internal regret $\intR_T$ are defined as
\begin{align*}
\extR_T = \sum_{t=1}^T \ellavg(y_t, p_t)  - \min_{i \in [N]} \sum_{t=1}^T \ell(y_t, a_{it}) &&
\intR_T  = \max_{i,j \in [N]} \sum_{t=1}^T p_{ti} \left( \ell(y_t, a_{it})  - \ell(y_t, a_{jt}) \right),
\end{align*}
where $ \ellavg(y, p) = \sum_{i=1}^N p_i \ell(y, a_{it}) $ is the expected loss.
\end{defn}
\paragraph{Calibration for online binary calibration}

For now, we focus on the $\ell_1$ norm, and we define the calibration error of a forecaster $\Fcal$ as
\begin{equation}
C_{T} = \sum_{i=0}^N \left| \rho_T(i/N) - \frac{i}{N} \right| \left( \frac{1}{T} \sum_{t=1}^T \Ind_{\{p_t = \frac{i}{N}\}} \right),
\end{equation}
where $\rho_T(p) = \frac{\sum_{t=1}^T y_t \Ind_{p_t = p}}{\sum_{t=1}^T \Ind_{p_t = p}}$ denotes the frequency at which event $y = 1$ occurred over the times when we predicted $p$.

We further define the calibration error when $\Fcal_j$ predicts $i/N$ as 
\begin{align*}
C^{(j)}_{T,i} & = \left| \rjt(i/N) - \frac{i}{N} \right| \left( \frac{1}{T_j} \sum_{t=1}^T \wsupj_{t,i} \right) 
\end{align*}
where $ \wsupj_{t,i} = \Ind\{p_t = \frac{i}{N} \cap \palg_t \in [\frac{j-1}{M},\frac{j}{M})\} $ is an indicator for the event that $\Fcal_j$ is triggered at time $t$ and predicts $i/N$.
Similarly, $ \Ind_{t,i} = \Ind\{p_t = i/N\} = \sum_{j=1}^M \wsupj_{t,i} $ indicates that $i/N$ was predicted at time $t$, and $T_j = \sum_{t=1}^T \sum_{i=0}^N \wsupj_{t,i}$ is the number of calls to $\Fcal_j$.
Also,
\begin{align*}
& \rjt(i/N) = \frac{\sum_{t=1}^T \wsupj_{t,i} y_t}{\sum_{t=1}^T \wsupj_{t,i}}
\end{align*}
is the empirical success rate for $\Fcal_j$. 

Note that with these definitions, we may write the calibration losses of $\Fcal_j$
as $ C^{(j)}_{T} = \sum_{i=0}^N C^{(j)}_{T,i}$.

\paragraph{Calibration for regression}

A sequence of forecasts $F_t$ achieves online quantile calibration for all $y \in \mathcal{Y}$ and all $p \in \mathcal{P}$,
$ \rho_T(y, p) \to p, $ a.s.~as $T \to \infty$, where
$$ \rho_T(y, p) = \dfrac{\sum_{t=1}^T \Ind_{y_t \leq y, F_t(y) = p}}{\sum_{t=1}^T \Ind_{F_t(y) = p}}$$
In other words, out of the times when the predicted probability $F_t(y')$ for $\{y_t \leq y'\}$ to be $p$, the event $\{y_t \leq y'\}$ holds a fraction $p$ of the time.

We also seek to quantify more precisely the calibration of Algorithm \ref{algo:recal}, specifically compare $\rho(y,p)$ with $p$. We define for this the quantity
\begin{align*}
C_{T,i}(y) = \left|  \rho_T(y, i/N) - \frac{i}{N} \right| \left( \frac{1}{T} \sum_{t=1}^T \Ind_{t,i} \right),
\end{align*}
and we define the calibration loss of Algorithm \ref{algo:recal} at $y$ as $C_{T}(y) = \sum_{i=0}^N C_{T,i}(y)$.


\paragraph{Proper losses}

The quality of probabilistic forecasts is evaluated using {\em proper} losses $\ell$. Formally, 
a loss $\ell(y, p)$ is proper if
$p \in \arg\min_{q \in \mathcal{P}} \Exp_{y \sim (p)} \ell(y, q) \; \forall p \in \mathcal{P}.$ 
An important proper loss for CDF predictions F is the continuous ranked probability score, defined as
$$\ell_\text{CRPS}(y, F) = \int_{-\infty}^\infty (F(z) - \mathbb{I}_{y \leq z})^2 dz.$$

\subsection{Assumptions}

We assume that each subroutine $\Fcal$ is an instance of a binary calibrated forecasting algorithm (e.g., the methods introduced in Chapter 4 in \cite{cesabianchi2006prediction}) that produce predictions in $[0,1]$ that are $(\e, \ell_2)$-calibrated and that $C^2_{T} \leq R_{T} + \e$ uniformly ($R_{T} = o(1)$ as $T \to \infty$; $T$ is the number of calls to instance $\Fcal_j$).
We also assume that for each $t$, the target $y_t$ lies in some bounded interval $\mathcal{Y}$ of $\mathbb{R}$ of length at most $B$.

\subsection{Online calibrated regression}

First, we look at algorithms for online calibrated regression (without covariates). Our algorithms leverage classical online binary calibration \citep{foster98asymptoticcalibration} as a subroutine.
Formally, \algorithmref{recal} 
partitions $[-\frac{B}{2},\frac{B}{2}]$ into $M$ intervals $\Ic = \{[\frac{-B}{2},\frac{-B}{2} + \frac{B}{M}), ..., [\frac{B}{2} - \frac{B}{M},\frac{B}{2}]\}$; each interval is associated with an instance of an online binary recalibration subroutine $\Fcal$ \citep{foster98asymptoticcalibration,cesabianchi2006prediction}. In order to compute $G_t(y \leq z)$, we invoke the subroutine $\Fcal_j$ associated with interval $I_j$ containing $z$.
After observing $y_t$, each $\Fcal_j$ observes whether $y_t$ falls in its interval and updates its state.

\begin{theorem}
    Let $\mathcal{Y}_\mathcal{I}$ be the set of upper bounds of the intervals $\mathcal{I}$ and let $\mathcal{P}_S$ be the output space of $\Fcal$. Algorithm \ref{algo:recal} achieves online calibration and for all $y \in \mathcal{Y}_\mathcal{I}, p \in \mathcal{P}_S$ we have $\rho_T(y,p) \to p$ a.s. as $T \to \infty$.
\end{theorem}

\begin{proof}
The above theorem follows directly from the construction of Algorithm \ref{algo:recal}: for each $y \in \mathcal{Y}$, we run an online binary calibration algorithm to target the event $y_t \leq y$. 

Specifically, note that for each $y \in \mathcal{Y}_\mathcal{I}$, the empirical frequency $\rho(y,p)$ reduces to the definition of the empirical frequency of a classical binary calibration algorithm targeting probability $p$ and the binary outcome that $y_t \leq y$. The output of the algorithm for $F_t(y)$ is also a prediction for the binary outcome $y_t \leq y$ produced by a classical onlne binary calibration algorithm. Thus, by construction, we have the desired result.
\end{proof}

Algorithms $\Fcal$ for online binary calibration are randomized. Our procedure needs to be randomized as well and this is a fundamental property of our task.

\begin{theorem}
    There does not exist a deterministic online calibrated regression algorithm that achieves online calibration.
\end{theorem}

\begin{proof}
This claim follows because we can encode a standard online binary calibration problem as calibrated regression. 
Specifically, given a non-randomized online calibrated regression algorithm, we could solve an online binary classification problem.
Suppose the adversary chooses a binary $y_t \in \{0,1\} \subseteq [0,1]$ that defines one of two classes. Then we can define an instance of calibrated regression with two buckets $[0,0.5)$ and $[0.5,1)$. We use the forecast $F_t(0.5)$ as our prediction for $y_t=0$ and one minus that as the prediction for 1. Then, the  error on the ratio $\rho_T(0.5,p)$ yields the definition of calibration in binary classification. If our deterministic online calibration regression algorithm works, then we have $\rho_T(0.5,p) \to p$, which means that the empirical ratio for the binary algorithm goes to the predicted frequency $p$ as well. But that would yield a deterministic algorithm for online binary calibration, which we know can't exist.
\end{proof}

\subsection{Proving the calibration of \algorithmref{recal}}

First, we will provide a proof of \lemmaref{calibration}; this proof holds for any norm $\ell_p$.


\begin{lemma}[Preserving calibration]
If each $\Fcal_j$ is $(\e, \ell_p)$-calibrated,
then \algorithmref{recal} is also $(\e, \ell_p)$-calibrated and the following bound holds uniformly over $T$:
\begin{align}
C_T \leq \sum_{j=1}^M \frac{T_j}{T} R_{T_j} + \e. \label{eqn:rate1}
\end{align}\vspace{-4mm}
\end{lemma}

\begin{proof}
Let $\wsupj_i = \sum_{t=1}^T \wsupj_{t,i}$ and note that $\sum_{t=1}^T \Ind_{t,i} = \sum_{j=1}^M \wsupj_i$. We may write
\begin{align*}
  C_{T,i}(y)
& = \frac{\sum_{t=1}^T \Ind_{t,i}}{T} \left|  \rho_T(y, i/N) - \frac{i}{N} \right|^p 
 = \frac{\sum_{j=1}^M \wsupj_i }{T} \left| \sum_{j=1}^M \frac{ \sum_{t=1}^T \wsupj_{t,i} o_{tj}}{\sum_{j=1}^M \wsupj_i }  - \frac{i}{N} \right|^p \\
& = \frac{\sum_{j=1}^M \wsupj_i}{T} \left| \sum_{j=1}^M \frac{ \wsupj_i \rjt(y, i/N) }{\sum_{j=1}^M \wsupj_i }  - \frac{i}{N} \right|^p 
 \leq \sum_{j=1}^M \frac{\wsupj_i}{T} \left| \rjt(y, i/N) - \frac{i}{N} \right|^p = \sum_{j=1}^M  \frac{T_j}{T} C^{(j)}_{T, i},
\end{align*}
where in the last line we used Jensen's inequality. 
Plugging in this bound in the definition of $C_T$, we find that 
\begin{align}
 C_T 
& = \sum_{i=1}^N  C_{T,i}
\leq \sum_{j=1}^M \sum_{i=1}^N \frac{T_j}{T} C^{(j)}_{T,i} 
 \leq \sum_{j=1}^M \frac{T_j}{T} R_{T_j} + \e, \nonumber
\end{align}
Since each $R_{T_j} \to 0$, \algorithmref{recal} will be $\e$-calibrated.
\end{proof}

\subsection{Recalibrated forecasts have low regret under the CRPS loss}


\begin{lemma}[Recalibration preserves accuracy]
Consider \algorithmref{recal} with
parameters $M \geq N > 1/\e$. Suppose that the $\Fcal$ are $(\e, \ell_2)$-calibrated.
Then the recalibrated $G_t$ a.s.~have vanishing $\ell_\text{CRPS}$-regret relative to $F_t$:
\begin{equation}
 \frac{1}{T} \sum_{t=1}^T \ell_\text{CRPS} (y_t , G_t) - \frac{1}{T} \sum_{t=1}^T \ell_\text{CRPS} (y_t , F_t) < NB R_T + \frac{2B}{N}.
\end{equation}
\end{lemma}

\begin{proof}
Our proof will rely on the following fact about any online calibration subroutine $\Fcal$. We start by formally establishing this fact.
\begin{fact}\label{fact:external_regret}
Let $\Fcal$ be an binary online calibration subroutine with actions $0, 1/N, ... 1$ whose $\ell_2$ calibration error $C^p_T$ is bounded by $R_T = o(T)$. Then the predictions $p_t$ from $\Fcal$ also minimize external regret relative to any single action $i/N$:
$$ \sum_{t=1}^T (p_t - y_t)^2 - (\frac{i}{N} - y_t)^2 \leq N R_T \text{ for all } i $$
\end{fact}
We refer the reader to  Lemma 4.4 in \cite{cesabianchi2006prediction} for a proof.

Next, we prove our main claim. We start with some notation.
Let $\Ic = \{[0,\frac{1}{M}), [\frac{1}{M}, \frac{2}{M}), ..., [\frac{M-1}{M},1]\}$ be a set of intervals that partition $[0,1]$ and let $I_j = [\frac{j-1}{M}, \frac{j}{M})$ be the $j$-th interval.
Also, for each $j$, we use $i_j$ denote the index $i \in [N]$ that is closest to $j$ in the sense of $|\frac{i_j}{N} - \frac{j}{M}| \leq \frac{1}{N}$. By our assumption that $M \geq N$, this index exists.

We begin our proof by from the definition of the CRPS regret:
\begin{align*}
&  \frac{1}{T} \sum_{t=1}^T \ell_\text{CRPS} (y_t , G_t) - \frac{1}{T} \sum_{t=1}^T \ell_\text{CRPS} (y_t , F_t) \\
& \;\; = \frac{1}{T} \sum_{t=1}^T \int_{-\infty}^\infty  (G_t(z) - \mathbb{I}_{y_t \leq z})^2 dz - \frac{1}{T} \sum_{t=1}^T \int_{-\infty}^\infty  (F_t(z) - \mathbb{I}_{y_t \leq z})^2 dz \\
& \;\; = \int_{-\infty}^\infty \frac{1}{T} \sum_{t=1}^T  \left[ (G_t(z) - \mathbb{I}_{y_t \leq z})^2 - (F_t(z) - \mathbb{I}_{y_t \leq z})^2 \right] dz  \\
& \;\; = \int_{z \in \mathcal{Y}} \frac{1}{T} \sum_{t=1}^T  \left[ (G_t(z) - \mathbb{I}_{y_t \leq z})^2 - (F_t(z) - \mathbb{I}_{y_t \leq z})^2 \right] dz \\
& \;\; =\int_{z \in \mathcal{Y}} \frac{1}{T} \sum_{t=1}^T  \left[ (G_t(z) - \mathbb{I}_{F_t(y_t) \leq F_t(z)})^2 - (F_t(z) - \mathbb{I}_{F_t(y_t) \leq F_t(z)})^2 \right] dz
\end{align*}
In the second-to-last line, we have used the fact that the forecasts have finite support, i.e., the $y_t$ live within a closed bounded set $\mathcal{Y}$.
In the last line, we replaced the event $y_t \leq z$ with $F_t(y_t) \leq F_t(z)$, which is valid because $F_t$ is monotonically increasing.

Let's now analyze the above integrand for one fixed value of $z$:
$$  \frac{1}{T} \sum_{t=1}^T  \left[ (G_t(z) - \mathbb{I}_{F_t(y_t) \leq F_t(z)})^2 - (F_t(z) - \mathbb{I}_{F_t(y_t) \leq F_t(z)})^2 \right]. $$
Since $F_t$ outputs a finite number of values in the set $\{0, \frac{1}{M}, ..., 1\}$, let $j/M$ denote the value $F_t(z) = j/M$ taken by $F_t$ at $z$.
Additionally, observe that $\mathbb{I}_{F_t(y_t) \leq \frac{j}{M}} = o_{tj}$, where $o_{tj}$ is the binary target variable given to $\Fcal_j$ at the end of step $t$. 
Finally, recall that when $F_t(z) = \frac{j}{M}$, we have defined $G_t(z)$ to be the output of $\Fcal_j$ at time $t$, which we denote as $G_{tj}$.
This yields the following expression for the above integrand for a fixed $z$:
$$  \frac{1}{T} \sum_{t=1}^T  \left[ (G_{tj} - o_{tj})^2 - (\frac{j}{M} - o_{tj})^2 \right]. $$
Next, recall that $i_j$ is the index $i \in [N]$ that is closest to $j$ in the sense of $|\frac{i_j}{N} - \frac{j}{M}| \leq \frac{1}{N}$. Recall also that $M \geq N$. Note that this implies
$$\ell_2(\frac{j}{M}, o_{tj}) \geq \ell_2(\frac{i_j}{M}, o_{tj}) + \frac{\partial \ell_2}{\partial p}(p,o_{tj})(\frac{j}{M}-\frac{i_j}{M}) \geq \frac{2}{N}.$$
Using this inequality, we obtain the following bound for our earlier integrand:
$$  \frac{1}{T} \sum_{t=1}^T  \left[ (G_{tj} - o_{tj})^2 - (\frac{i_j}{N} - o_{tj})^2 \right] + \frac{2}{N}. $$
Crucially, this expression is precisely the {\em external regret} of recalibration subroutine $\Fcal_j$ relative to the fixed action $\frac{i_j}{N}$ and measured in terms of the L2 loss. By Fact \ref{fact:external_regret}, we know that this external regret is bounded by $N R_T$. Since this bound holds pointwise for any value of $z$, we can plug it into our original integral to obtain a bound on the CRPS regret:
\begin{align*}
& \int_{z \in \mathcal{Y}} \frac{1}{T} \sum_{t=1}^T  \left[ (G_t(z) - \mathbb{I}_{F_t(y_t) \leq F_t(z)})^2 - (F_t(z) - \mathbb{I}_{F_t(y_t) \leq F_t(z)})^2 \right] dz \\
& \;\; \leq \int_{z \in \mathcal{Y}} \left[ N R_T + \frac{2}{N} \right] dz \\
& \;\; \leq NB R_T + \frac{2B}{N} 
\end{align*}
In the last line, we used the fact that the integration is over a finite set $\mathcal{Y}$ whose measure is bounded by $B > 0$. This establishes the main claim of this proposition.
\end{proof}

\subsection{Correctness of \algorithmref{recal}}

We now prove our main result about the correctness of \algorithmref{recal}.

\setcounter{theorem}{0}
\begin{theorem}
Let $\Fcal$ be an $(\ell_1, \epsilon/3B)$-calibrated online subroutine with resolution $N \geq 3B/\epsilon$. 
and let $\ell$ be a proper loss satisfying the assumptions of \lemmaref{accuracy}. Then \algorithmref{recal} with parameters $\Fcal$ and $N$ is an $\epsilon$-accurate online recalibration algorithm for the loss $\ell$.
\end{theorem}

\begin{proof}
It is easy to show that \algorithmref{recal} is $(\ell_1, \e/3B)$-calibrated by the same argument as Lemma 1 (see the next section for a formal proof). By Lemma 4, its regret w.r.t. the raw $\palg_t$ tends to $< 3B/N < \e$. Hence, the theorem follows.
\end{proof}

\subsection{Calibration implies no internal regret}

Here, we show that a calibrated forecaster also has small internal regret relative to any bounded proper loss \citep{kuleshov2017estimating}.

\setcounter{lemma}{0}

\begin{lemma}
If $\ell$ is a bounded proper loss, then an $\e$-calibrated $\Fcal$
a.s.~has a small internal regret w.r.t.~$\ell$ and satisfies uniformly over time $T$ the bound
\begin{align}
\intR_{T} = \max_{ij} \sum_{t=1}^T \Ind_{p_t = i/N} \left( \ell(y_t, i/N)  - \ell(y_t, j/N) \right) \leq 2 B (R_T + \e).
\end{align}
\end{lemma}

\begin{proof}

Let $T$ be fixed for the rest of this proof.
Let $\Ind_{ti} = \Ind_{p_t = i/N}$ be the indicator of $\Fcal$ outputting prediction $i/N$ at time $t$, let $T_i = \sum_{t=1}^T \Ind_{ti}$ denote the number of time $i/N$ was predicted,  and let
$$ \intR_{T, ij} = \sum_{t=1}^T \Ind_{ti} \left( \ell(y_t, i/N)  - \ell(y_t, j/N) \right) $$
denote the gain (measured using the proper loss $\ell$) from retrospectively switching all the plays of action $i$ to $j$. This value forms the basis of the definition of internal regret (Section 2).

Let $T(i,y) = \sum_{t=1}^T \Ind_{ti} \Ind\{y_t = y\}$ denote the total number of $i/N$ forecasts at times when $y_t = y \in \{0,1\}$. Observe that we have
\begin{align*}
T(i,y) 
& = \sum_{t=1}^T \Ind_{ti} \Ind\{y_t = y\} 
= \frac{\sum_{t=1}^T \Ind_{ti} \Ind\{y_t = y\} }{T_i} T_i
= \frac{\sum_{t=1}^T \Ind_{ti} \Ind\{y_t = y\} }{\sum_{t=1}^T \Ind_{ti}} T_i \\
& = q(i,y) T_i + T_i \left( \frac{\sum_{t=1}^T \Ind_{ti} \Ind\{y_t = y\} }{\sum_{t=1}^T \Ind_{ti}} - q(i,y) \right) \\
& = q(i,y) T_i + T_i \left( \rho_T(i/N) - i/N \right),
\end{align*}
where $q(i,y) = i/N$ if $y=1$ and $1-i/N$ if $y=0$. The last equality follows using some simple algebra after adding and subtracting one inside the parentheses in the second term.

We now use this expression to bound $\intR_{T, ij}$:
\begin{align*}
\intR_{T, ij}
& = \sum_{t=1}^T \Ind_{ti} \left( \ell(y_t, i/N)  - \ell(y_t, j/N) \right) \\
& = \sum_{y \in \{0,1\}} T(i,y) \left( \ell(y, i/N)  - \ell(y, j/N) \right) \\
& \leq \sum_{y \in \{0,1\}} q(i,y) T_i \left( \ell(y, i/N)  - \ell(y, j/N) \right) + \sum_{y \in \{0,1\}} B T_i \left| \rho_T(i/N) - i/N \right| \\
& \leq 2B T_i \left| \rho_T(i/N) - i/N \right|,
\end{align*}
where in the first inequality, we used $\ell(y, i/N)  - \ell(y, j/N) \leq \ell(y, i/N)  \leq B$, and in the second inequality we used the fact that $\ell$ is a proper loss.

Since internal regret equals $\intR_{T} = \max_{i,j} \intR_{T, ij}$, we have
\begin{align*}
\intR_{T}
& \leq \sum_{i=1}^N \max_{j} \intR_{T, ij} 
 \leq 2B \sum_{i=0}^N T_i \left| \rho(i/N) - i/N \right| 
 \leq 2 B ( R_T + \e ).
\end{align*}

\end{proof}

\subsection{Impossibility of recalibrating non-proper losses}

We conclude the appendix by explaining why non-proper losses cannot be calibrated \citep{kuleshov2017estimating}.

\begin{theorem}
If $\ell$ is not proper, then there is no recalibration algorithm w.r.t.~$\ell$.
\end{theorem}

\begin{proof}
If $\ell$ is not proper, there exist a $p'$ and $q$ such that $\Exp_{y \sim \text{Ber}(p')} \ell(y, q)  < \Exp_{y \sim \text{Ber}(p')} \ell(y, p') $.

Consider a sequence $y_t$ for which $y_t \sim \text{Ber}(p')$ for all $t$. Clearly the prediction of a calibrated forecaster $p_t$ much converge to $p'$ and the average loss will approach $\ell(y, p')$. This means that we cannot recalibrate the constant predictor $p_t = q$ without making its loss $\ell(y, q)$ higher. We thus have a forecaster that cannot be recalibrated with respect to $\ell$.
\end{proof}

\section{Low regret relative to baseline classifiers}\label{app:regret}

Here, we show that a calibrated forecaster also has small regret relative to any bounded proper loss if we use a certain construction that combines our algorithm with a baseline forecaster. This extends our previous construction to more general settings.

\subsection{Recalibration construction}

\paragraph{Setup}

We start with an online forecaster $F$ that outputs uncalibrated forecasts $\palg_t$ at each step; these forecasts are fed into a {\em recalibrator} such that the resulting forecasts $p_t$ are calibrated and have low regret relative to the baseline forecasts $\palg_t$. 

Formally, at every step $t=1,2,...$ we have:
  \begin{algorithmic}[1]
    \STATE Forecaster $F$ predicts $\palg_t$.
    \STATE A recalibration algorithm produces a calibrated forecast $p_t$ based on $\palg_t$.
    \STATE Nature reveals label $y_t$
    \STATE Based on $x_t, y_t$, we update the recalibration algorithm and optionally update $H$.
  \end{algorithmic}

\paragraph{Notation}  

We define a discretization $V$ of the space of forecasts. We assume that the forecasts live in a compact set $\Delta$ and we define a triangulation of $\Delta$, i.e., a partition into a set of simplices
such that any two simplices intersect in either a common face, common vertex,
or not at all. Let $V$ be the vertex set of this triangulation, and let
$V (p)$ be the set of corners for this simplex. 

Note that each distribution $p$ can be uniquely written as a weighted average of its neighboring vertices, $V (p)$. For $v \in V (p)$, we define the test functions $w_v(p)$ to be
these linear weights, so they are uniquely defined by the linear equation $p=\sum_{v \in V(p)} w_v(p) v$.
We also define the discretization to be sufficiently small: given a target precision $\epsilon > 0$ we define the discretization such that for all $f_1, f_2$ in the same simplex we have $|| f_1 - f_2 || < \epsilon$.

\subsection{Recalibration algorithm}

We are going to define a general meta-algorithm that follows a construction in which we run multiple instances of our calibrated forecasting algorithms over the inputs of $F$.

More formally, we take the aforementioned partition of the space of forecasts of $\Delta$ of $F$ and we associate each simplex with an instance of 
our calibration algorithm $\Fcal$ (using the same $\Delta$ and discretization $V$). In order to compute $\palg_t$, we invoke the subroutine $\Fcal_j$ associated with simplex $I_j$ containing $\palg_t$ (with ties broken arbitrarily).
After observing $y_t$, we pass it to $\Fcal_j$.

The resulting procedure produces valid calibrated estimates because each $\Fcal_j$ is a calibrated subroutine. More importantly the new forecasts do not decrease the predictive performance of $F$, as measured by a proper loss $\ell$.
In the remainder of this section, we establish these facts formally.

\subsection{Theoretical analysis}

\paragraph{Notation}

Our task is to produce calibrated forecasts. Intuitively, we say that a forecast $F_t$ is calibrated if for every $y' \in \mathcal{Y}$, the probability $F_t(y')$ on average matches the frequency of the event $\{ y = y' \}$.
We formalize this by introducing the ratio
\begin{equation}
    \rho_T(p) = \dfrac{\sum_{t=1}^T y_t \cdot \Ind_{p_t = p}}{\sum_{t=1}^T \Ind_{p_t = p}}
\end{equation}
Intuitively, we want $ \rho_T(p) \to p, $ a.s.~as $T \to \infty$ for all $y$.
In other words, out of the times when the predicted probability for $y_t$ is $p$, the average $y_t$ look like $p$.

The quality of probabilistic forecasts is evaluated using {\em proper} losses $\ell$. Formally, 
a loss $\ell(y, p)$ is proper if
$p \in \arg\min_{q \in \mathcal{P}} \Exp_{y \sim (p)} \ell(y, q) \; \forall p \in \mathcal{P}.$ 
An example in binary classification is the log-loss $\ell_\text{log}(y,p) = y\log(p) + (1-y)\log(1-p)$. We will assume that the loss is bounded by $B > 0$ .

We measure calibration a calibration error $C_T$. Our algorithms will output discretized probabilities; hence we define the error relative to a set of possible predictions $V$
\begin{equation}
    C_T  = \sum_{p \in V} \left| \rho_T(p) - p \right| \left( \frac{1}{T} \sum_{t=1}^T \Ind_{\{p_t = p\}} \right). 
\end{equation}


\subsubsection{A helper lemma}

In order to establish the correctness of our recalibration procedure, we need to start with a helper lemma. This lemma shows that if forecasts are calibrated, then they have small internal regret.

\begin{lemma}
If $\ell$ is a bounded proper loss, then an $(\e, \ell_1)$-calibrated $\Fcal$
a.s.~has a small internal regret w.r.t.~$\ell$ and satisfies uniformly over time $T$ the bound
\begin{align}
\intR_{T} = \max_{ij} \sum_{t=1}^T \Ind_{p_t = p_i} \left( \ell(y_t, p_i)  - \ell(y_t, p_j) \right) \leq 2 B (R_T + \e).
\end{align}
\end{lemma}

\begin{proof}

Let $T$ be fixed for the rest of this proof.
Let $\Ind_{ti} = \Ind_{p_t = p_i}$ be the indicator of $\Fcal$ outputting prediction $p_i$ at time $t$, let $T_i = \sum_{t=1}^T \Ind_{ti}$ denote the number of time $i/N$ was predicted,  and let
$$ \intR_{T, ij} = \sum_{t=1}^T \Ind_{ti} \left( \ell(y_t, p_i)  - \ell(y_t, p_j) \right) $$
denote the gain (measured using the proper loss $\ell$) from retrospectively switching all the plays of action $i$ to $j$. This value forms the basis of the definition of internal regret.

Let $T(i,y) = \sum_{t=1}^T \Ind_{ti} \Ind\{y_t = y\}$ denote the total number of $p_i$ forecasts at times when $y_t = y$. Observe that we have
\begin{align*}
T(i,y) 
& = \sum_{t=1}^T \Ind_{ti} \Ind\{y_t = y\} 
= \frac{\sum_{t=1}^T \Ind_{ti} \Ind\{y_t = y\} }{T_i} T_i
= \frac{\sum_{t=1}^T \Ind_{ti} \Ind\{y_t = y\} }{\sum_{t=1}^T \Ind_{ti}} T_i \\
& = q(i,y) T_i + T_i \left( \frac{\sum_{t=1}^T \Ind_{ti} \Ind\{y_t = y\} }{\sum_{t=1}^T \Ind_{ti}} - q(i,y) \right) \\
& = q(i,y) T_i + T_i \left( \rho_T(p_i) - p_i \right),
\end{align*}
where $q(i,y) = p_i(y)$. The last equality follows using some simple algebra after adding and subtracting one inside the parentheses in the second term.

We now use this expression to bound $\intR_{T, ij}$:
\begin{align*}
\intR_{T, ij}
& = \sum_{t=1}^T \Ind_{ti} \left( \ell(y_t, p_i)  - \ell(y_t, p_j) \right) \\
& = \sum_{y} T(i,y) \left( \ell(y, p_i)  - \ell(y, p_j) \right) \\
& \leq \sum_{y} q(i,y) T_i \left( \ell(y, p_i)  - \ell(y, p_j) \right) + B T_i \left| \rho_T(p_i) - p_i \right| \\
& \leq B T_i \left| \rho_T(p_i) - p_i \right|,
\end{align*}
where in the first inequality, we used $\ell(y, p_i)  - \ell(y, p_j) \leq \ell(y, p_i)  \leq B$, and in the second inequality we used the fact that $\ell$ is a proper loss.

Since internal regret equals $\intR_{T} = \max_{i,j} \intR_{T, ij}$, we have
\begin{align*}
\intR_{T}
& \leq \sum_{i=1}^N \max_{j} \intR_{T, ij} 
 \leq 2B \sum_{i=0}^N T_i \left| \rho(i/N) - p_i \right| 
 \leq 2 B ( R_T + \e ).
\end{align*}

\end{proof}

\subsection{Recalibrated forecasts have low regret relative to uncalibrated forecasts}

Next, we use the above result to prove that the forecasts recalibrated using the above construction have low regret relative to the baseline uncalibrated forecasts.


\begin{lemma}[Recalibration preserves accuracy]
Let $\ell$ be a bounded proper loss such that $\ell(y_t, p) \leq \ell(y_t, p_j) + B\epsilon$ whenever $||p - p_j|| \leq \epsilon$.
Then the recalibrated $p_t$ a.s.~have vanishing $\ell$-loss regret relative to $\palg_t$ and we have uniformly:
\begin{equation}
\frac{1}{T} \sum_{t=1}^T \ell (y_t , p_t) - \frac{1}{T} \sum_{t=1}^T \ell(y_t , \palg_t)  < \frac{B}{\epsilon} \sum_{j=1}^M \frac{T_j}{T} R_{T_j} + 3B\e.
\end{equation}
\end{lemma}

\begin{proof}
By the previous lemma, we know that an algorithm whose calibration error is bounded by $R_T = o(1)$ also minimizes internal regret at a rate of $2BR_T$, and thus external regret at a rate of $2BR_T / \epsilon$.

Next, let us use $\Ind_{j,t}$ to indicate that $\Fcal_j$ was called at time $t$. 
We establish our main claim as follows:
\begin{align*}
& \frac{1}{T} \sum_{t=1}^T  \ell (y_t , p_t) - \frac{1}{T} \sum_{t=1}^T \ell (y_t , \palg_t) \\
& \;\; = \frac{1}{T} \sum_{t=1}^T \left( \sum_{j=1}^M \left( \ell (y_t , p_t) - \ell (y_t , \palg_t) \right) \Ind_{j,t} \right) \\
& \;\; < \frac{1}{T} \sum_{t=1}^T \left( \sum_{j=1}^M \left( \ell (y_t , p_t) - \ell (y_t , p_j) \right) \Ind_{j,t} + B\epsilon \right) \\
& \;\; \leq \frac{1}{\epsilon} B \sum_{j=1}^M \frac{T_j}{T} R_{T_j} + 3B\epsilon,
\end{align*}
where $R_{T_j}$ is a bound on the calibration error of $\Fcal_j$ after $T_j$ plays. 

In the first two inequality, we use our assumption on the loss $\ell$.
The last inequality follows because $\Fcal_j$ minimizes external regret w.r.t.~the constant action $p_j$ at a rate of $BR_{T_j}/\epsilon$.
\end{proof}

\subsection{Proving that calibration holds}

We want to also give a proof that the recalibration construction described above yields calibrated forecasts.

\begin{lemma}
If each $\Fcal_j$ is $(\e, \ell_p)$-calibrated,
then the combined algorithm is also $(\e, \ell_p)$-calibrated and the following bound holds uniformly over $T$:
\begin{align}
C_T \leq \sum_{j=1}^M \frac{T_j}{T} R_{T_j} + \e. \label{eqn:rate}
\end{align}\vspace{-4mm}
\end{lemma}

\begin{proof}
Let $M = |V|$.
Let $\wsupj_i = \sum_{t=1}^T \wsupj_{t,i}$ where $\wsupj_{t,i} = \Ind \{p_t = p_j \cap \palg_t = p_j\}$ and note that $\sum_{t=1}^T \Ind_{t,i} = \sum_{j=1}^M \wsupj_i$. Let also $\rjt(p_i) = \frac{\sum_{t=1}^T \wsupj_{t,i} y_t}{\sum_{t=1}^T \wsupj_{t,i}}$. We may write
\begin{align*}
  C_{T,i} 
& = \frac{\sum_{t=1}^T \Ind_{t,i}}{T} \left|  \rho_T(p_i) - p_i \right|
 = \frac{\sum_{j=1}^M \wsupj_i }{T} \left| \sum_{j=1}^M \frac{ \sum_{t=1}^T \wsupj_{t,i} y_t}{\sum_{j=1}^M \wsupj_i }  - p_i \right| \\
& = \frac{\sum_{j=1}^M \wsupj_i}{T} \left| \sum_{j=1}^M \frac{ \wsupj_i \rjt(p_i) }{\sum_{j=1}^M \wsupj_i }  - p_i \right|
 \leq \sum_{j=1}^M \frac{\wsupj_i}{T} \left| \rjt(p_i) - p_i \right| = \sum_{j=1}^M  \frac{T_j}{T} C^{(j)}_{T, i},
\end{align*}
where $C^{(j)}_{T,i} = \left| \rjt(p_i) - p_i \right| \left( \frac{1}{T_j} \sum_{t=1}^T \wsupj_{t,i} \right)$ and in the last line we used Jensen's inequality. 
Plugging in this bound in the definition of $C_T$, we find that 
\begin{align}
 C_T 
& = \sum_{i=1}^N  C_{T,i}
\leq \sum_{j=1}^M \sum_{i=1}^N \frac{T_j}{T} C^{(j)}_{T,i} 
 \leq \sum_{j=1}^M \frac{T_j}{T} R_{T_j} + \e, \nonumber
\end{align}
Since each $R_{T_j} \to 0$, the full procedure will be $\e$-calibrated.
\end{proof}

Recall that $R_T$ denotes the rate of convergence of the calibration error $C_T$.
For most online calibration subroutines $\Fcal$,
$R_{T} \leq f(\e)/\sqrt{T}$ for some $f(\e)$.
In such cases, we can further bound the calibration error in the above lemma as
$$
\sum_{j=1}^M \frac{T_j}{T} R_{T_j} \leq \sum_{j=1}^M \frac{\sqrt{T_j}f(\e)}{T} \leq \frac{f(\e)}{\sqrt{ \e T}}. 
$$
In the second inequality, we set the $T_j$ to be equal. 
Thus, our recalibration procedure introduces an overhead of
$ \frac{1}{\sqrt{\e}} $
in the convergence rate of the calibration error $C_T$ and of the regret relative to a baseline forecaster in the earlier lemma.

\newpage

\section{Applications: Decision-making}\label{app:applications}

Next, we complement our results with a formal characterization of some benefits of calibration. We are interested in decision-making settings where we wish to estimate the value of a function $v : \mathcal{Y} \times \mathcal{A} \times \mathcal{X} \to \mathbb{R}$ over a set of outcomes $\mathcal{Y}$, actions $\mathcal{A}$, and features $\mathcal{X}$. Note that the function $v$ could be a loss $\ell(y,a,x)$ that quantifies the error of an action $a \in \mathcal{A}$ in a state $x \in \mathcal{X}$ given outcome $y \in \mathcal{Y}$.

We assume that given $x$, the agent chooses an action $a(x)$ according to a decision-making process. This could be an action $a(x) = \arg \min_a \mathbb{E}_{y \sim H(x)} [ \ell(y, a,x) ]$ that minimizes a loss that are trying to estimate, but any outcome is possible.
The agent then relies on a predictive model $H$ of $y$ to estimate the future values $v (y, a,x)$ for the decision $a(x)$ :
\begin{align}
v(x) & = \mathbb{E}_{y \sim H(x)} [ v(y, a(x),x) ].
\end{align}
We study $v(y,a,x)$ that are monotonically non-increasing or non-decreasing in $y$. Examples include linear utilities $u(a,x) \cdot y + c(a,x)$ or their monotone transformations. 

\paragraph{Expectations under calibrated models}

If $H$ was a perfect predictive model, we could estimate expected values of outcomes perfectly. In practice, inaccurate models can yield imperfect decisions. Surprisingly, our analysis shows that in many cases, calibration (a much weaker condition that having a perfectly specified model $H$) is sufficient to correctly estimate the value of various outcomes.

Surprisingly, our guarantees can be obtained with a weak condition---quantile calibration.
Additional requirements are the non-negativity and monotonicity of $v$.
Our result is a concentration inequality that shows that estimates of $v$ are unlikely to exceed the true $v$ on average.

\begin{theorem}
\label{apdx:thm:dist_calib_bound_app}
Let $M$ be a quantile calibrated model as in and
let $v(y, a, x)$ be a monotonic value function.
Then for any sequence $(x_t, y_t)_{t=1}^T$ and $r > 0$, we have:
\begin{equation}
    \label{apdx:eqn:dist_calib_bound1}
    \lim_{T \to \infty} \frac{1}{T} \sum_{t=1}^T \mathbb{I} \left[ v(y_t, a(x_t), x_t) \geq r v(x_t)) \right] \leq 1 / r
\end{equation}
\end{theorem}

\begin{proof}

Recall that $M(x)$ is a distribution over $\mathcal{Y}$, with a density $p_x$, a quantile function $Q_x$, and a cdf $F_x$.
Note that for any $x$ and $s \in (0,1)$ and $y' \leq F_x^{-1}(1-s)$ we have:
\begin{align*}
v(x)  
& = \int v(x, y, a(x)) q_x(y) dy \\
& \geq \int_{y \geq y'} v(x, y, a(x)) q_x(y) dy \\
& \geq v(x, y', a(x)) \int_{y \geq y'} q_x(y) dy \\
& \geq s v(x, y', a(x))
\end{align*}

The above logic implies that whenever $v(x)   \leq s v(x, y, a)$, we have $y \geq F_x^{-1}(1-s)$ or $F_x(y) \geq (1-s)$. Thus, we have for all $t$,
\begin{align*}
\mathbb{I}\{  v(x_t)   \leq s v(x_t, y_t, a_t) \} \leq \mathbb{I}\{  F_{x_t}(y_t) \geq (1-s) \}.
\end{align*}
Therefore, we can write
\begin{align*}
\frac{1}{T} \sum_{t=1}^T \mathbb{I}\{  v(x_t)   \leq s v(x_t, y_t, a_t) \} \leq \frac{1}{T} \sum_{t=1}^T \mathbb{I}\{  F_{x_t}(y_t) \geq (1-s) \} = s + o(T),
\end{align*}
where the last equality follows because $M$ is calibrated. Therefore, the claim holds in the limit as $T \to \infty$ for $r = 1/s$. 
The argument is similar if $v$ is monotonically non-increasing. In that case, we can show that whenever $y' > F_x^{-1}(s)$, we have $v(x)  \geq s v(x, y', a(x))$. Thus, whenever $v(x)   \leq s v(x, y, a)$, we have $y \leq F_x^{-1}(s)$ or $F_x(y) \leq s$. Because, $F_x$ is calibrated, we again have that
\begin{align*}
\frac{1}{T} \sum_{t=1}^T \mathbb{I} \{ v(x_t)   \leq s v(x_t, y_t, a_t) \} \leq \sum_{t=1}^T \mathbb{I} \{ F_{x_t}(y_t) < s \} = s + o(T),
\end{align*}
and the claim holds with $r = 1/s$. 
\end{proof}

Note that this statement represents an extension of Markov inequality. 
Note also that this implies the same result for a distribution calibrated model, since distribution calibration implies quantile calibration.

\section{Experiments on UCI Benchmarks}
\label{apdx:uci_expt}
The existing UCI datasets~\citep{Dua2019UCI} used in our experiments hold a Creative Commons Attribution 4.0 International (CC BY 4.0) license. 
\paragraph{Computational resources.} Our experiments were conducted on a laptop with 2.3 GHz 8-Core Intel Core i9 processor and 32 GB 2667 MHz DDR4 RAM. The code and datasets take 16MB memory. 

\paragraph{Detailed setup.}
Our dataset consists of input and output pairs $\{x_t, y_t\}_{t=1}^{T}$ where $T$ is the size of the dataset. 
We simulate a stream of data by sending batches of data-points $\{x_t, y_t\}_{t=nt'+1}^{n(t'+1)}$ to our model, where $t'$ is the time-step and $n$ is the batch-size. This simulation is run for $\left \lceil{T/n}\right \rceil $ time-steps. For each batch, Bayesian ridge regression is fit to the data and the recalibrator is trained. 
We set $N=20$ in the recalibrator and use a batch size of $n=10$ for all experiments except for the Aquatic Toxicity dataset~\ref{fig:daphnia-aquatic-toxicity} where we used $n=5$. The calibration is evaluated at levels $[0.2, 0.4, 0.5, 0.6, 0.8]$. 

\section{Experiments on Bayesian optimization}
\label{apdx:bayes_opt}
Bayesian optimization attempts to find the global minimum $x^\star = \arg \min_{x \in \mathcal{X}} f(x)$ of an unknown function $f:\mathcal{X} \to \mathbb{R} $ over an input space $\mathcal{X} \subseteq \mathbb{R}^D$. 
We are given an initial labeled dataset $x_t, y_t \in \mathcal{X} \times \mathbb{R}$ for $t = 1, 2, ..., N$ of i.i.d. realizations of random variables $X,Y \sim P$. At every time-step $t$, we use uncertainties from the probabilistic model $\mathcal{M}:\mathcal{X} \to (\mathbb{R} \to [0, 1])$ of $f$ to select the next data-point $x_{next}$ and iteratively update the model $\mathcal{M}$.  Algorithm~\ref{algo:plain-bo} outlines this procedure. Since the black-box function evaluation can be expensive, the objective of Bayesian optimization in this context is to find the minima (or maxima) of this function while using a small number of function evaluations.

\paragraph{Computational resources.} Our experiments were conducted on a laptop with 2.3 GHz 8-Core Intel Core i9 processor and 32 GB 2667 MHz DDR4 RAM. The code and datasets take 16MB memory.

\paragraph{Detailed setup. } We use online calibration to improve the uncertainties estimated by the model $\mathcal{M}$. Following~\citet{Deshpande2021Calibrated}, we use Algorithm~\ref{algo:calibrate} to recalibrate the model $\mathcal{M}$. Since the dataset size is small, we use the \textsc{CREATESPLITS} function to generate leave-one-out cross-validation splits of our dataset $\mathcal{D}.$  We train the base model on train-split and use this to obtain probabilistic forecast for data in the test-split. We collect these predictions on all test-splits to form our recalibration dataset and use Algorithm~\ref{algo:recal} to perform calibration.

Following~\citet{Deshpande2021Calibrated}, we perform calibrated Bayesian optimization as detailed in Algorithm~\ref{algo:calibrated-bo}. Specifically, we recalibrate the base model $\mathcal{M}$ after every step in Bayesian optimization. We build on the GpyOpt library~\citep{gpyopt2016} for Bayesian optimization that holds the BSD 3-clause license. 

We use some popular benchmark functions to evaluate the performance of Bayesian optimization. We initialize the Bayesian optimization with 3 randomly chosen data-points. We use the Lower Confidence Bound (LCB) acquisition function to select the data-point $x_t$ and evaluate a potentially expensive function $f$ as $x_t$ to obtain $y_t$.  At any given time-step $T$, we have the dataset $\mathcal{D}_T = \{x_t, y_t\}_{t=1}^{T}$ collected iteratively. 

In Figure~\ref{fig:bayes-opt}, we see that using online calibration of uncertainties from $\mathcal{M}$ allows us to reach a lower minimum or find the same minimum with a smaller number of steps with Bayesian optimization.

\begin{figure}[h]
\centering     
\subfigure[SixHumpCamel]{\label{ewa-recalibrator-sixhumpcamel2}\includegraphics[width=0.32\linewidth]{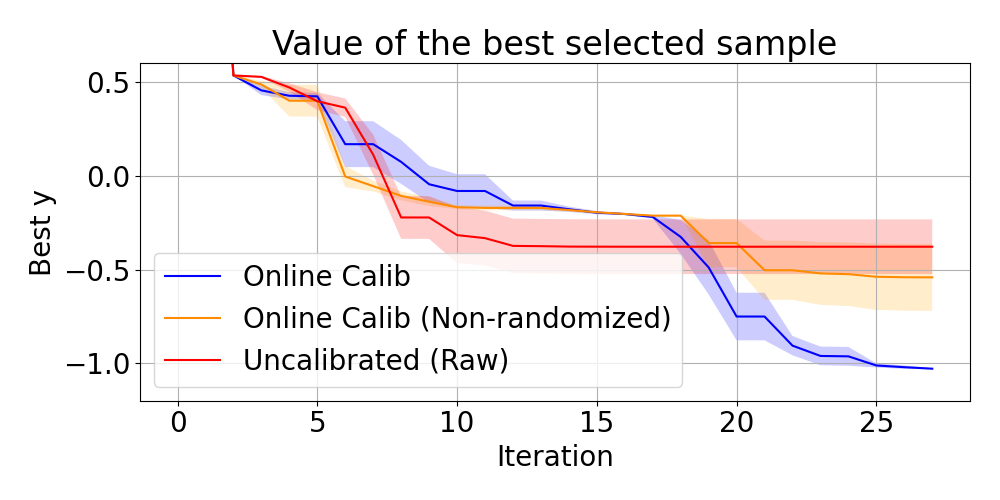}}
\subfigure[Beale]{\label{beale2}\includegraphics[width=0.32\linewidth]{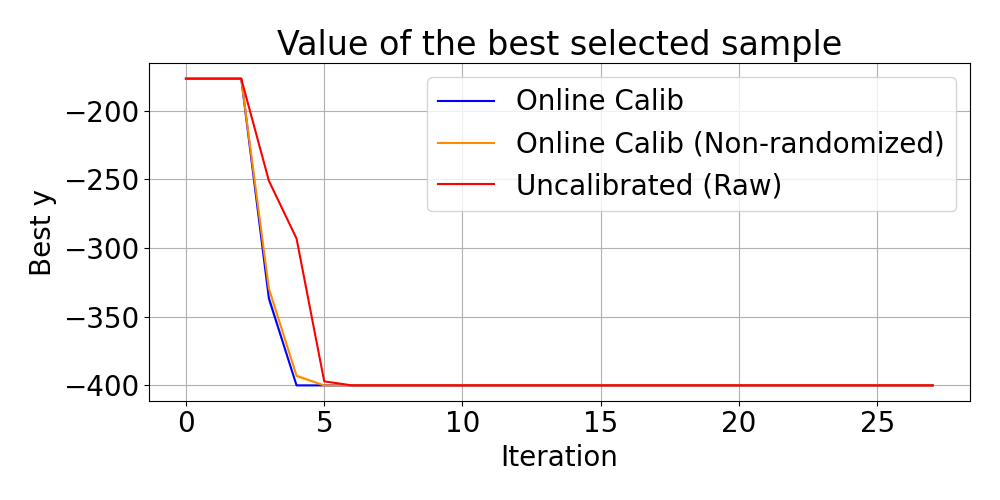}}
\subfigure[Mccormick]{\label{Mccormick2}\includegraphics[width=0.32\linewidth]{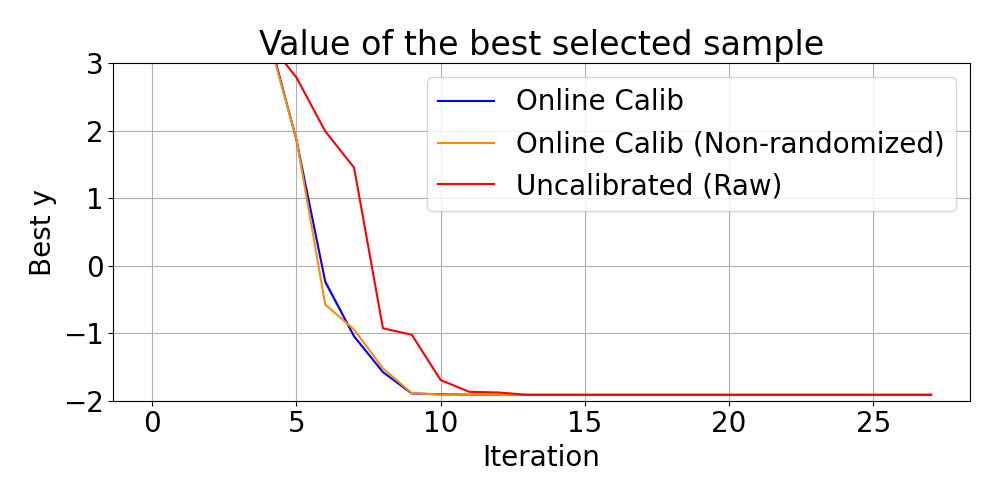}}
\caption{Online Calibration Improves Bayesian optimization}
\label{fig:bayes-opt}
\end{figure}

\begin{algorithm}[H]
  \caption{Bayesian Optimization}
  \label{algo:plain-bo}
  \begin{algorithmic}[1]
    \STATE Initialize base model $\mathcal{M}$ with data $\mathcal{D}=\{x_t, y_t\}_{t=0}^{M}$.
    \FOR {$n=1,2,...,T$:}
    \STATE $x_{\textrm{next}}$ = $\arg \max_{x \in \mathcal{X}}(\textrm{Acquisition}(x, \mathcal{R} \circ \mathcal{M}))$.
    \STATE $y_{\textrm{next}}$ = $f(x_{\textrm{next}})$.
    \STATE $\mathcal{D}$ = $\mathcal{D} \bigcup \{(x_{\textrm{next}}, y_{\textrm{next}})\}$
    \STATE Update model $\mathcal{M}$ with data $\mathcal{D}$
    \ENDFOR
  \end{algorithmic}
\end{algorithm}

\begin{algorithm}[H]
  \caption{Calibrated Bayesian Optimization~\citep{Deshpande2021Calibrated}}
  \label{algo:calibrated-bo}
  \begin{algorithmic}[1]
    \STATE Initialize base model $\mathcal{M}$ with data $\mathcal{D}=\{x_t, y_t\}_{t=0}^{M}$.
    \STATE $\mathcal{R} \gets \textsc{Calibrate}(\mathcal{M,\mathcal{D}})$.
    \FOR {$n=1,2,...,T$:}
    \STATE $x_{\textrm{next}}$ = $\arg \max_{x \in \mathcal{X}}(\textrm{Acquisition}(x, \mathcal{R} \circ \mathcal{M}))$.
    \STATE $y_{\textrm{next}}$ = $f(x_{\textrm{next}})$.
    \STATE $\mathcal{D}$ = $\mathcal{D} \bigcup \{(x_{\textrm{next}}, y_{\textrm{next}})\}$
    \STATE Update model $\mathcal{M}$ with data $\mathcal{D}$
    \STATE $\mathcal{R} \gets \textsc{Calibrate}(\mathcal{M,\mathcal{D}})$
    \ENDFOR
  \end{algorithmic}
\end{algorithm}

\begin{algorithm}[H]
  \caption{$\textsc{Calibrate}$~\citep{Deshpande2021Calibrated}}
  \label{algo:calibrate}
  \begin{algorithmic}[1]
    \REQUIRE Base model $\mathcal{M}$, Dataset $\mathcal{D}=\{x_t, y_t\}_{t=0}^{N}$
    \STATE Train a base model $\mathcal{M}$ on training dataset $\{x_t, y_t\}_{t=0}^{N}$.
    \STATE Initialize recalibration dataset $\mathcal{D_{\textrm{recal}}} = \phi$
     \STATE $S = \textsc{CreateSplits({D})}$
    \FOR {$(\mathcal{D_{\textrm{train}}}, \mathcal{D_{\textrm{test}}})$ in $S$:}
    \STATE $\mathcal{D_{\textrm{train}}}=$ Train Dataset $\{x_t, y_t\}_{t=0}^{M}$ in split $s$.
    \STATE $\mathcal{D_{\textrm{test}}}=$ Test Dataset $\{x_t, y_t\}_{t=0}^{L}$ in split $s$.
    \STATE $\mathcal{D_{\textrm{train}}}=\textsc{TrainSplit}(s), \mathcal{D_{\textrm{test}}}=\textsc{TestSplit}(s)$
    \STATE Train base model $\mathcal{M'}$ on dataset $\mathcal{D_{\textrm{train}}}$
    \STATE Compute CDF dataset $\{[M'(x_t)](y_t)\}_{t=1}^{M}$ from dataset $\mathcal{D_{\textrm{test}}}$
    \STATE $\mathcal{D_{\textrm{recal}}} = \mathcal{D_{\textrm{recal}}} \bigcup \{[\mathcal{M'}(x_t)](y_t), y_t\}_{t=1}^{M}$
    \ENDFOR
    \STATE Train recalibrator model $\mathcal{R}$ on the recalibration dataset $\mathcal{D_{\textrm{recal}}}$ using Algorithm~\ref{algo:recal}
    \STATE Return ($\mathcal{R}$)
  \end{algorithmic}
\end{algorithm}


\section{Comparison to prior work}

\subsection{Comparing to \citet{kuleshov2017estimating}}

While Kuleshov and Ermon [2017] and the Calibeating technique focus on binary classification, we study regression. We want to emphasize that moving from calibration to regression is non-trivial and significantly more involved than generalizing the scoring rule from CDFs to point forecasts. The regression setting is significantly harder than classification, and requires (1) non-trivial thinking about how to define calibration and (2) algorithms and analyses that are substantially different than in classification.


The classical definition of calibration (of the times when I predict $p$, binary event holds $p$ \% of the time) does not easily carry over to regression. In fact, an “easier” version of regression is multi-class calibration (imagine the continuous label $y$ is discretized), and even that is PPAD-hard (Hazan and Kakade, 2012).

Thus, most work on regression studies marginal notions of calibration: a $p$-\% confidence interval contains the label $p$-\% of the time (note how we omit the “when I predict $p$” part). Still, maintaining this in a non-IID setting is non-trivial. One well-known method is ACI (Gibbs and Candes, 2021), but it does not admit regret guarantees. We define a novel and slightly stronger notion of marginal calibration (which has elements of conditional calibration; see Eqn 2), and we provide regret guarantees.

Also, quantifying and minimizing regret is itself non-trivial. This requires defining a suitable notion of regret that is compatible with our definition of calibration. We use the CRPS and CDF recalibration as measures of regret and calibration, respectively.


While our method superficially resembles that of Kuleshov and Ermon (2017) (and Calibeating, which is the same algorithm) in that we partition an interval and run simple subroutines in each sub-interval, the analysis is significantly different, especially the part about minimizing regret. Superficially, while that proof takes (1/2)-page in Kuleshov and Ermon, ours is about 2 pages long and is substantially different.

Note also that we provide a significant number of additional results that strengthen our core work: a generalized Markov inequality that guarantees our method is able to accurately estimate losses, an analysis of confidence intervals, and an application to online decision-making and Bayesian optimization.

\subsection{Comparing to \citet{Deshpande2021Calibrated}}

Note that the focus and the methods of both papers are different. Our work makes more theoretical contributions around the feasibility of defining and maintaining good calibration and regret in an online non-IID regression setting. The work by Deshpande and Kuleshov is mainly empirical: it applies methods from IID regression (e.g., Kuleshov and Ermon, ICML2018) and additional heuristics to obtain the best possible empirical results on classification. Please note also that their paper is an unpublished manuscript.

We adopt a similar setup to Deshpande and Kuleshov in our experiments because the setting is useful and inherently non-IID. However, because our work is more theoretical, our experiments are not as extensive as those of Deshpande and Kuleshov (whose entire paper is mostly experimental). That said, our non-randomized baseline (orange line) is effectively equivalent to the IID algorithm used in Deshpande and Kuleshov (it simply maintains marginal calibration by counting frequencies in bins), and we outperform that baseline in our experiments by virtue of designing specialized non-IID algorithms.

Lastly, while the paper by Deshpande and Kuleshov has a lemma on online decision-making, ours holds in the online non-IID setting, while theirs is only IID.
